\documentclass[final,3p]{elsarticle}

\usepackage{amsmath}
\usepackage{amssymb,amsbsy}
\usepackage{amsthm}
\usepackage{pstricks}
\usepackage{subfigure,color,graphicx,psfrag}
\DeclareGraphicsExtensions{.pdf,.png,.gif,.jpg}
\usepackage{cases}

\journal{}
\renewcommand*{\today}{April 25, 2013}

\newtheorem{theorem}{Theorem}
\newtheorem{proposition}{Proposition}
\newtheorem{corollary}{Corollary}
\newtheorem{lemma}{Lemma}

\newtheorem{remark}{Remark}

\renewenvironment{proof}{\noindent{\bf Proof:}}{\qed\medskip}
\newenvironment{proofref}[1]{\noindent{\bf Proof (of~#1):}}{\qed\medskip}
\newcommand{\refapx}[1]{\ref{#1}}

\renewcommand{\leq}{\leqslant}
\renewcommand{\geq}{\geqslant}
\renewcommand{\phi}{\varphi}
\renewcommand{\hat}{\widehat}
\renewcommand{\tilde}{\widetilde}

\newcommand{\be}{\boldsymbol{e}}

\newcommand{\bp}{\boldsymbol{p}}

\newcommand{\bu}{\boldsymbol{u}}
\newcommand{\bv}{\boldsymbol{v}}
\newcommand{\bw}{\boldsymbol{w}}
\newcommand{\bx}{\boldsymbol{x}}

\newcommand{\bphi}{\boldsymbol{\phi}}

\newcommand{\cA}{\mathcal{A}}

\newcommand{\cX}{\mathcal{X}}

\newcommand{\dd}{\mbox{\rm d}}

\newcommand{\E}{\mathbb{E}}
\newcommand{\Law}{\mathbb{L}}
\newcommand{\N}{\mathbb{N}}

\newcommand{\Z}{\mathbb{Z}}
\newcommand{\Prob}{\mathbb{P}}
\newcommand{\R}{\mathbb{R}}

\renewcommand{\epsilon}{\varepsilon}
\providecommand{\eps}{\epsilon}

\newcommand{\argmin}{\mathop{\mathrm{argmin}}}

\newcommand{\norm}[1][\cdot]{\ensuremath{\left\Arrowvert #1 \right\Arrowvert}}
\providecommand{\Norm}[1]{\left\lVert#1\right\rVert}
\newcommand{\Zp}{\N^*}

\newcommand{\abs}[1][\cdot]{\ensuremath{\left| #1 \right|}}
\providecommand{\Abs}[1]{\left\lvert#1\right\rvert}

\newcommand{\indicator}[1]{\mathbb{I}_{#1}}
\newcommand{\indicatorB}[1]{\mathbb{I}_{\left\{#1\right\}}}
\newcommand{\eqdef}{\triangleq}

\newcommand{\Var}{\textrm{Var}}

\newcommand{\sgn}{\mathrm{sgn}}

\newcommand{\iid}{i.\@i.\@d.\@ }
\newcommand{\ie}{{i.e.}}
\newcommand{\eg}{{e.g.}}

\newcommand{\documentstatus}{paper}

\newlength{\minipagewidth}
\newcommand{\bookbox}[1]{
\par\medskip\noindent
\framebox[\textwidth]{
\begin{minipage}{\minipagewidth}
{#1}
\end{minipage} } \par\medskip }

\begin{document}

\begin{frontmatter}
\title{Adaptive and optimal online linear regression on $\ell^1$-balls}

\author[ensparis]{S{\'e}bastien Gerchinovitz\corref{cor1}\fnref{fn1}}
\ead{sebastien.gerchinovitz@ens.fr}

\author[ibmdublin]{Jia Yuan Yu}
\ead{jiayuanyu@ie.ibm.com}

\cortext[cor1]{Corresponding author}
\fntext[fn1]{This research was
    carried out within the INRIA project CLASSIC hosted by {\'E}cole
    Normale Sup{\'e}rieure and CNRS.}

\address[ensparis]{{\'E}cole Normale Sup{\'e}rieure, 45 rue d'Ulm, 75005 Paris, France}
    
\address[ibmdublin]{IBM Research, Damastown Technology Campus, Dublin 15, Ireland}

\begin{abstract}
We consider the problem of online linear regression on individual sequences. The goal in this paper is for the forecaster to output sequential predictions which are, after $T$ time rounds, almost as good as the ones output by the best linear predictor in a given $\ell^1$-ball in $\R^d$. We consider both the cases where the dimension~$d$ is small and large relative to the time horizon $T$. We first present regret bounds with optimal dependencies on $d$, $T$, and on the sizes $U$, $X$ and $Y$ of the $\ell^1$-ball, the input data and the observations. The minimax regret is shown to exhibit a regime transition around the point $d = \sqrt{T} U X / (2 Y)$. Furthermore, we present efficient algorithms that are adaptive, \ie, that do not require the knowledge of $U$, $X$, $Y$, and $T$, but still achieve nearly optimal regret bounds.
\end{abstract}


\end{frontmatter}


%%%%%%%%%%%%%%%%%%%%%%%%%%%%%%%%%%%%%%%%%%%%%%%%%%%%%%%%%%%%%%%%%%%%%
\section{Introduction}
%%%%%%%%%%%%%%%%%%%%%%%%%%%%%%%%%%%%%%%%%%%%%%%%%%%%%%%%%%%%%%%%%%%%%

In this {\documentstatus}, we consider the problem of online linear regression
against arbitrary sequences of input data and observations, with the objective
of being competitive with respect to the best linear predictor in an $\ell^1$-ball of arbitrary
radius. This extends the task of convex aggregation. 
We consider both low- and high-dimensional input data. Indeed, in a large number of contemporary problems, the available data can be high-dimensional---the dimension of each data
point is larger than the number of data points. Examples include analysis of DNA sequences, collaborative filtering, astronomical data analysis, and cross-country growth regression. In such high-dimensional problems, performing linear regression on an $\ell^1$-ball of small diameter may be helpful if the best linear predictor is sparse.
Our goal is, in both low and high dimensions, to provide online linear regression algorithms along with bounds on $\ell^1$-balls that characterize their robustness to worst-case scenarios.

\subsection{Setting}
\label{sec:set}

We consider the online version of linear regression, which unfolds as
follows.  First, the environment chooses a sequence of observations
$(y_t)_{t \geqslant 1}$ in $\R$ and a sequence of input vectors
$(\bx_t)_{t \geqslant 1}$ in $\R^d$, both initially hidden from the
forecaster.  At each time instant $t \in \mathbb{N}^* =
\{1,2,\ldots\}$, the environment reveals the data $\bx_t \in
\R^d$; the forecaster then gives a prediction $\hat y_t \in \R$;
the environment in turn reveals the observation $y_t \in \R$;
% each linear forecaster $u \in \R^d$ incurs the loss $\bigl(y_t - u
% \cdot x_t\bigl)^2$
and finally, the forecaster incurs the square loss $(y_t - \widehat{y}_t)^2$.
The dimension $d$ can be either small or large relative to the number $T$ of
time steps: we consider both cases.\\

In the sequel, $\bu \cdot \bv$ denotes the standard inner product between $\bu,\bv \in \R^d$, and we set $\Norm{\bu}_{\infty} \eqdef \max_{1 \leq j \leq d} |u_j|$ and $\Norm{\bu}_1 \eqdef \sum_{j=1}^d |u_j|$. The $\ell^1$-ball of radius $U>0$ is the following bounded subset of~$\R^d$:
\begin{align*}
  B_1(U) \eqdef \left\{\bu \in \R^d: \Norm{\bu}_1 \leq U\right\}.
\end{align*}
Given a fixed radius $U>0$ and a time horizon $T \geq 1$, the goal of the forecaster is to predict almost as well as the best linear forecaster in the reference set $\bigl\{\bx \in \R^d \mapsto \bu \cdot \bx \in \R: \bu \in B_1(U)\bigr\}$, \ie, to minimize the regret on $B_1(U)$ defined by
\begin{align*}
  \sum_{t=1}^T (y_t - \hat y_t)^2 - \min_{\bu \in B_1(U)}
  \left\{ \sum_{t=1}^T (y_t - \bu \cdot \bx_t)^2 \right\}.
\end{align*}

We shall present algorithms along with bounds on their regret that
hold uniformly over all sequences\footnote{Actually our results hold whether $(\bx_t,y_t)_{t \geq 1}$ is generated by an oblivious environment or a non-oblivious opponent since we consider deterministic forecasters.} $(\bx_t,y_t)_{1 \leq t \leq T}$ such that $\Norm{\bx_t}_{\infty} \leq X$ and $|y_t| \leq Y$ for all $t=1,\ldots,T$, where $X,Y>0$.
These regret bounds depend on four important quantities: $U$, $X$, $Y$, and $T$, which may be known or unknown to the forecaster.

\subsection{Contributions and related works}
\label{sec:intro-contributions}

In the next paragraphs we detail the main contributions of this {\documentstatus} in view of related works in online linear regression.\\

Our first contribution (Section~\ref{sec:rates}) consists of a minimax analysis of online linear regression on $\ell^1$-balls in the arbitrary sequence setting. We first provide a refined regret bound expressed in terms of $Y$, $d$, and a quantity $\kappa = 
\sqrt{T}UX/(2dY)$.  This quantity $\kappa$ is used to distinguish two regimes: we show a distinctive regime transition\footnote{In high dimensions (\ie, when $d > \omega T$, for some absolute constant $\omega>0$), we do not observe this
  transition (cf.\ Figure~\ref{fig:upperbound}).} at $\kappa = 1$ or $d = \sqrt{T}UX/(2Y)$.  Namely, for
$\kappa < 1$, the regret is of the order of $d Y^2 \kappa$ (proportional to $\sqrt{T}$), whereas it is of the order of $d Y^2 \ln \kappa$ (proportional to $\ln T$) for $\kappa > 1$.

The derivation of this regret bound partially relies on a Maurey-type argument used under various forms with {i.i.d.}\ data, \eg, in \cite{Nem-00-TopicsNonparametric,Tsy-03-OptimalRates,BuNo08SeqProcedures,ShSrZh-10-Sparsifiability} (see also \cite{Yan-04-BetterPerformance}). We adapt it in a straightforward way to the deterministic setting. Therefore, this is yet another technique that can be applied to both the stochastic and individual sequence settings.

Unsurprisingly, the refined regret bound mentioned above matches the optimal risk bounds for stochastic settings\footnote{For example, $(\bx_t,y_t)_{1 \leq t \leq T}$ may be \iid, or $\bx_t$ can be deterministic and $y_t = f(\bx_t) + \eps_t$ for an unknown function $f$ and an \iid sequence $(\eps_t)_{1 \leq t \leq T}$ of Gaussian noise.}
\cite{BiMa-01-GaussianMS,Tsy-03-OptimalRates} (see also \cite{RaWaYu-09-MinimaxSparseRegression}). Hence,
linear regression is just as hard in the stochastic setting as in the
arbitrary sequence setting. Using the standard online to batch conversion, we make the latter statement more precise by establishing a lower bound for all $\kappa$ at least of the order of $\sqrt{\ln d}/d$. This lower bound extends those of \cite{CB99AnalysisGradientBased,KiWa97EGvsGD}, which only hold for small $\kappa$ of the order of $1/d$. \\

The algorithm achieving our minimax regret bound is both computationally inefficient and non-adaptive (\ie, it requires prior knowledge of the quantities $U$, $X$, $Y$, and $T$ that may be unknown in practice). Those two issues were first overcome by \cite{AuCeGe02Adaptive} via an automatic tuning termed \emph{self-confident} (since the forecaster somehow trusts himself in tuning its parameters). They indeed proved that the self-confident $p$-norm algorithm with $p=2\ln d$ and tuned with $U$ has a cumulative loss $\hat{L}_T=\sum_{t=1}^T (y_t-\hat{y}_t)^2$ bounded by
\begin{align*}
\hat{L}_T & \leq L_T^* + 8 U X \sqrt{(e \ln d) \, L_T^*} + (32 e \ln d) \, U^2 X^2 \\
& \leq 8 U X Y \sqrt{e T \ln d} + (32 e \ln d) \, U^2 X^2~,
\end{align*}
where $L_T^* \eqdef \min_{\{\bu \in \R^d:\norm[\bu]_1 \leq U\}} \sum_{t=1}^T (y_t - \bu \cdot \bx_t)^2 \leq T Y^2$. This algorithm is efficient, and our lower bound in terms of $\kappa$ shows that it is optimal up to logarithmic factors in the regime $\kappa \leq 1$ without prior knowledge of $X$, $Y$, and $T$.

Our second contribution (Section~\ref{sec:algo1}) is to show that similar adaptivity and efficiency properties can be obtained via exponential weighting. We consider a variant of the $\textrm{EG}^{\pm}$ algorithm \cite{KiWa97EGvsGD}. The latter has a manageable computational complexity and our lower bound shows that it is nearly optimal in the regime $\kappa \leq 1$. However, the $\textrm{EG}^{\pm}$ algorithm requires prior knowledge of $U$, $X$, $Y$, and $T$. To overcome this adaptivity issue, we study a modification of the $\textrm{EG}^{\pm}$ algorithm that relies on the variance-based automatic tuning of \cite{CeMaSt07SecOrder}. The resulting algorithm -- called \emph{adaptive $\textrm{EG}^{\pm}$ algorithm} -- can be applied to general convex and differentiable loss functions. When applied to the square loss, it yields an algorithm of the same computational complexity as the $\textrm{EG}^{\pm}$ algorithm that also achieves a nearly optimal regret but without needing to know $X$, $Y$, and $T$ beforehand.

Our third contribution (Section~\ref{sec:loss-Lip}) is a generic technique called \emph{loss Lipschitzification}. It transforms the loss functions $\bu \mapsto (y_t-\bu \cdot \bx_t)^2$ (or $\bu \mapsto \big|y_t-\bu \cdot \bx_t\big|^{\alpha}$ if the predictions are scored with the $\alpha$-loss for a real number $\alpha \geq 2$) into Lipschitz continuous functions. We illustrate this technique by applying the generic adaptive $\textrm{EG}^{\pm}$ algorithm to the modified loss functions. When the predictions are scored with the square loss, this yields an algorithm (the LEG algorithm) whose main regret term slightly improves on that derived for the adaptive $\textrm{EG}^{\pm}$ algorithm without Lipschtizification. The benefits of this technique are clearer for loss functions with higher curvature: if $\alpha > 2$, then the resulting regret bound roughly grows as $U$ instead of a naive~$U^{\alpha/2}$.

Finally, in Section~\ref{sec:algo2}, we provide a simple way to achieve minimax regret uniformly over all $\ell^1$-balls $B_1(U)$ for $U>0$. This method aggregates instances of an algorithm that requires prior knowledge of $U$. For the sake of simplicity, we assume that $X$, $Y$, and $T$ are known, but explain in the discussions how to extend the method to a fully adaptive algorithm that requires the knowledge neither of $U$, $X$, $Y$, nor $T$.

\bigskip

This {\documentstatus} is organized as follows.  In Section~\ref{sec:rates}, we
establish our refined upper and lower bounds in terms of the intrinsic quantity $\kappa$.  In
Section~\ref{sec:algo1}, we present an efficient and adaptive algorithm --- the adaptive $\textrm{EG}^{\pm}$ algorithm with or without loss Lipschitzification --- that
achieves the optimal regret on $B_1(U)$ when $U$ is known. In Section~\ref{sec:algo2}, we use an aggregating strategy to achieve an optimal regret uniformly over all $\ell^1$-balls
$B_1(U)$, for $U\!>\!0$, when $X$, $Y$, and $T$ are known.
Finally, in Section~\ref{sec:dis},
we discuss as an extension a fully automatic algorithm that requires no prior knowledge of $U$, $X$, $Y$, or $T$. Some proofs and additional tools are postponed to the appendix.

%%%%%%%%%%%%%%%%%%%%%%%%%%%%%%%%%%%%%%%%%%%%%%%%%%%%%%%%%%%%%%%%%%%%%
\section{Optimal rates}\label{sec:rates}
%%%%%%%%%%%%%%%%%%%%%%%%%%%%%%%%%%%%%%%%%%%%%%%%%%%%%%%%%%%%%%%%%%%%%

In this section, we first present a refined upper bound on the minimax regret on
$B_1(U)$ for an arbitrary $U>0$. In
Corollary~\ref{cor:upperbound}, we express this upper bound in terms
of an intrinsic quantity $\kappa \eqdef \sqrt{T} U X/(2dY)$. The
optimality of the latter bound is shown in
Section~\ref{sec:lowerbounds}.

We consider the following definition to avoid any ambiguity. We call \emph{online forecaster} any sequence $F=(\tilde{f}_t)_{t \geq 1}$ of functions such that $\tilde{f}_t:\R^d \times (\R^d \times \R)^{t-1} \to \R$ maps at time $t$ the new input $\bx_t$ and the past data $(\bx_s,y_s)_{1 \leq s \leq t-1}$ to a prediction $\tilde{f}_t\bigl(\bx_t;(\bx_s,y_s)_{1 \leq s \leq t-1}\bigr)$. Depending on the context, the latter prediction may be simply denoted by $\tilde{f}_t\bigl(\bx_t)$ or by $\hat{y}_t$.

\subsection{Upper bound}
\label{sec:upperbounds}

\begin{theorem}[Upper bound]
  \label{thm:upperbound}
  Let $d, T \in \Zp$, and $U, X, Y >0$. The minimax regret on $B_1(U)$
  for bounded base predictions and observations satisfies
  \begin{align*}
    & \inf_F \sup_{\Norm{\bx_t}_\infty \leq X,\; \Abs{y_t}\leq Y} \Biggl\{ \sum_{t=1}^T (y_t - \hat y_t)^2 - \inf_{\Norm{\bu}_1 \leq U} \sum_{t=1}^T (y_t - \bu \cdot \bx_t)^2 \Biggr\} \\
    & \quad \leq \left\{\begin{array}{ll}
        3UXY\sqrt{2 T \ln (2d)} & \textrm{if} \quad U < \frac{Y}{X} \sqrt{\frac{\ln(1+2d)}{T \ln 2}}~, \\
        26 \, UXY\sqrt{T \ln \left(1+\frac{2dY}{\sqrt{T}UX}\right)} & \textrm{if} \quad \frac{Y}{X} \sqrt{\frac{\ln(1+2d)}{T \ln 2}} \leq U \leq \frac{2 d Y}{\sqrt{T} X}~, \\
        32 \, d Y^2 \ln\!\left(1+\frac{\sqrt{T} U X}{d Y}\right) + d Y^2 & \textrm{if} \quad U > \frac{2dY}{X \sqrt{T}}~,
      \end{array} \right.
  \end{align*}
  where the infimum is taken over all forecasters $F$ and where the
  supremum extends over all sequences $(\bx_t,y_t)_{1 \leq t \leq T}
  \in (\R^d \times \R)^T$ such that $|y_1|, \ldots, |y_T| \leq Y$ and
  $\Norm{\bx_1}_{\infty}, \ldots, \Norm{\bx_T}_{\infty} \leq~X$.
\end{theorem}

Theorem~\ref{thm:upperbound} improves the bound of
\cite[Theorem~5.11]{KiWa97EGvsGD} for the
$\textrm{EG}^{\pm}$ algorithm.
First, our bound
depends logarithmically---as opposed to linearly---on $U$ for $U > 2dY/(\sqrt{T}X)$.
Secondly, it is smaller by a factor ranging from $1$ to $\sqrt{\ln d}$ when
\begin{equation}
  \label{eqn:smallU}
  \frac{Y}{X} \sqrt{\frac{\ln(1+2d)}{T \ln 2}} \leq U \leq \frac{2 d Y}{\sqrt{T} X}~.
\end{equation}
Hence, Theorem~\ref{thm:upperbound} provides a partial answer to a question\footnote{The authors of \cite{KiWa97EGvsGD} asked: ``For large $d$ there is a significant gap between the upper and lower bounds. We would like to know if it possible to improve the upper bounds by eliminating the $\ln d$ factors.''} raised in \cite{KiWa97EGvsGD} about the gap of $\sqrt{\ln(2d)}$ between the upper and lower bounds.\\

Before proving the theorem (see below), we state the following immediate corollary. It expresses the upper bound of Theorem~\ref{thm:upperbound} in
terms of an intrinsic quantity $\kappa \eqdef \sqrt{T} U
X/(2dY)$ that relates $\sqrt{T} U X/(2Y)$ to the ambient dimension~$d$.

\vspace{0.2cm}
\begin{corollary}[Upper bound in terms of an intrinsic quantity]
  \label{cor:upperbound}
  Let $d, T \in \N^*$, and $U, X, Y >0$. The upper bound of
  Theorem~\ref{thm:upperbound} expressed in terms of $d$, $Y$, and the
  intrinsic quantity $\kappa \eqdef \sqrt{T} U X / (2 d Y)$ reads:
  \begin{align*}
    & \inf_F \sup_{\Norm{\bx_t}_\infty \leq X,\; \Abs{y_t}\leq Y} \Biggl\{ \sum_{t=1}^T (y_t - \hat y_t)^2 - \inf_{\Norm{\bu}_1 \leq U} \sum_{t=1}^T (y_t - \bu \cdot \bx_t)^2 \Biggr\} \\
    & \quad \leq \left\{\begin{array}{ll}
        6 \, d Y^2 \kappa \sqrt{2\ln(2d)} & \textrm{if} \quad \kappa < \frac{\sqrt{\ln(1+2d)}}{2 d \sqrt{\ln 2}}~, \\
        52 \, d Y^2 \kappa \sqrt{\ln(1+1/\kappa)} & \textrm{if} \quad \frac{\sqrt{\ln(1+2d)}}{2 d \sqrt{\ln 2}} \leq \kappa \leq 1~, \\
        32 \, d Y^2 \bigl(\ln(1+2\kappa)+1\bigr) & \textrm{if} \quad \kappa > 1~.
      \end{array} \right.
  \end{align*}
\end{corollary}

The parametrization by $(d,Y,\kappa)$ helps to unify the different upper bounds of Theorem~\ref{thm:upperbound}: on both regimes $\kappa \leq 1$ and $\kappa > 1$, the regret bound scales as $d Y^2$, the only difference lies in the dependence in $\kappa$ (linear versus logarithmic).

The upper bound of Corollary~\ref{cor:upperbound} is shown in~Figure~\ref{fig:upperbound}. Observe that, in low dimension (Figure~\ref{fig:1a}), a clear transition from a regret of the order of $\sqrt{T}$ to one of $\ln T$ occurs at $\kappa=1$. This transition is absent for high dimensions: for $d \geq \omega T$, where $\omega \eqdef \big(32 (\ln(3)+1)\big)^{-1}$, the regret bound $32 \, d Y^2 \bigl(\ln(1+2\kappa)+1\bigr)$ is worse than a trivial bound of $TY^2$ when $\kappa \geq 1$.

\begin{figure}[ht]
\begin{center}
  \begin{tabular}{ll}
    \subfigure[High dimension $d \geq \omega T$.]{\label{fig:1b}\includegraphics[scale=0.34]{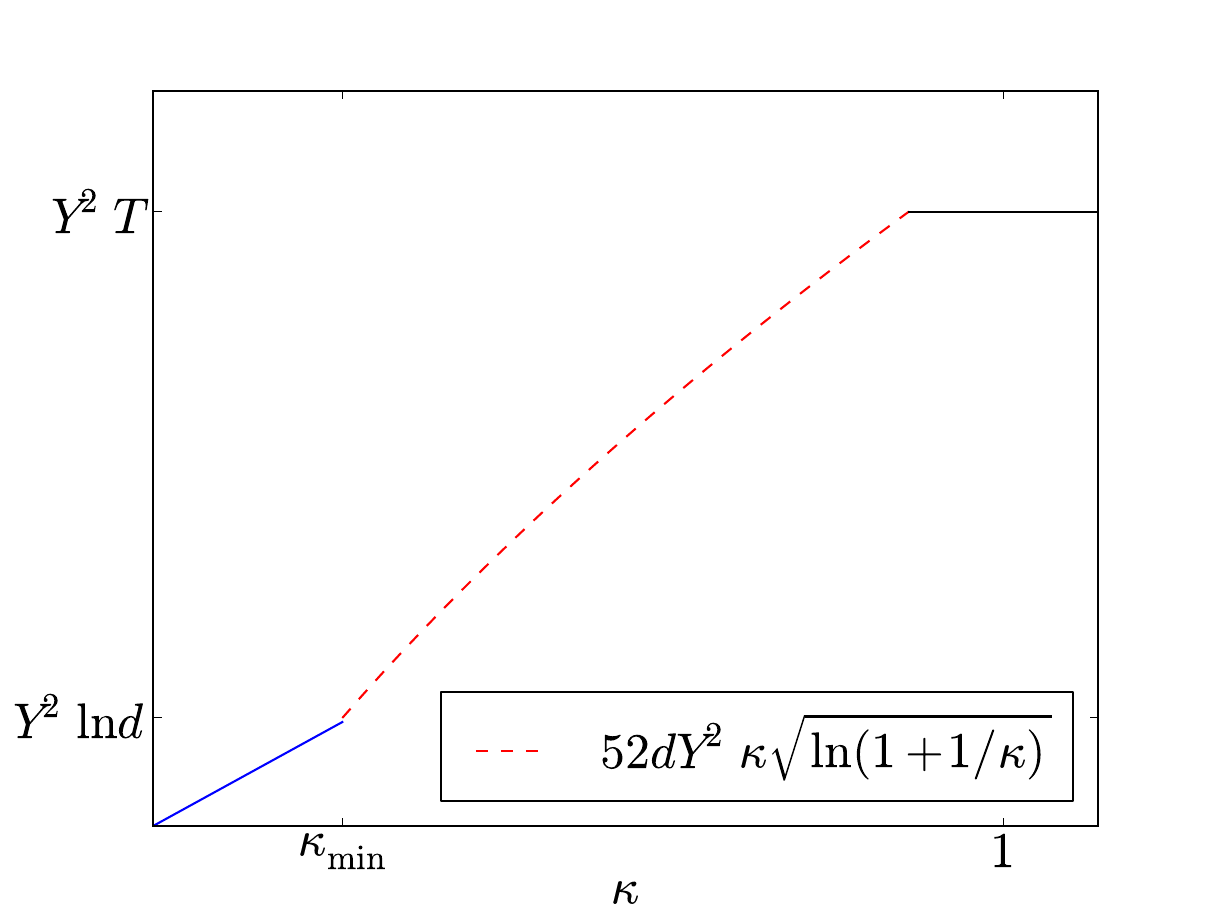}}
    & \subfigure[Low dimension $d < \omega T$.]{\label{fig:1a}\includegraphics[scale=0.34]{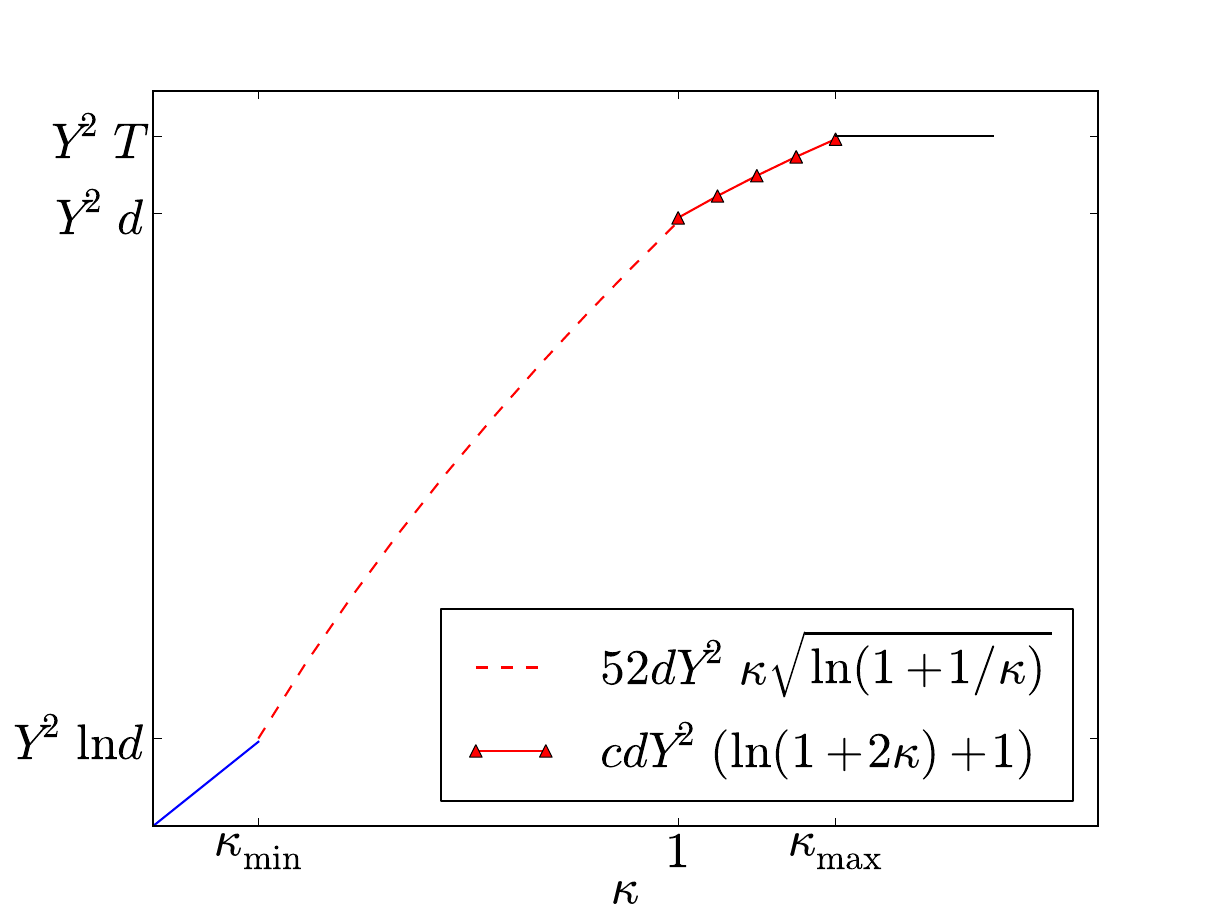}}
   \end{tabular}
\end{center}
	\vspace{-0.3cm}
  \caption{\label{fig:upperbound} The regret bound of Corollary~\ref{cor:upperbound} over $B_1(U)$
    as a function of $\kappa = \sqrt{T}UX/(2dY)$. The constant $c$ is
    chosen to ensure continuity at $\kappa=1$, and $\omega \eqdef \big(32 (\ln(3)+1)\big)^{-1}$. We define: $\kappa_{\min} = \sqrt{\ln(1+2d)}/(2d\sqrt{\ln 2})$ and $\kappa_{\max} = (e^{(T/d-1)/c}-1)/2$.}
\end{figure}

\vspace{0.5cm}
We now prove Theorem~\ref{thm:upperbound}. The main part of the proof relies on a Maurey-type argument. Although this argument was used in the stochastic setting \cite{Nem-00-TopicsNonparametric,Tsy-03-OptimalRates,BuNo08SeqProcedures,ShSrZh-10-Sparsifiability}, we adapt it to the deterministic setting. This is yet another technique that can be applied to both the stochastic and individual sequence settings.

\vspace{0.5cm}
\begin{proofref}{Theorem~\ref{thm:upperbound}}
First note from Lemma~\ref{lem:chapL1-upper-lemma} in \refapx{sec:chapL1-lemmas} that the minimax regret on $B_1(U)$ is upper bounded\footnote{As proved in Lemma~\ref{lem:chapL1-upper-lemma}, the regret bound \eqref{eqn:chapL1-upperbound-1} is achieved either by the $\textrm{EG}^{\pm}$ algorithm, the algorithm $\textrm{SeqSEW}^{B,\eta}_{\tau}$ of \cite{Ger-11colt-SparsityRegretBounds} (we could also get a slightly worse bound with the sequential ridge regression forecaster \cite{AzWa01RelativeLossBounds,Vo01CompetitiveOnline}), or the trivial null forecaster.} by
\begin{equation}
\label{eqn:chapL1-upperbound-1}
\min\left\{3UXY\sqrt{2 T \ln (2d)}, \, 32 \, d Y^2 \ln\!\left(1+\frac{\sqrt{T} U X}{d Y}\right) + d Y^2 \right\}~.
\end{equation}
Therefore, the first case $U < \frac{Y}{X} \sqrt{\frac{\ln(1+2d)}{T \ln 2}}$ and the third case $U>\frac{dY}{X \sqrt{T}}$ are straightforward.\\

\noindent
Therefore, we assume in the sequel that $\frac{Y}{X} \sqrt{\frac{\ln(1+2d)}{T \ln 2}} \leq U \leq \frac{2 d Y}{\sqrt{T} X}$. \\
We use a Maurey-type argument to refine the regret bound \eqref{eqn:chapL1-upperbound-1}. This technique was used under various forms in the stochastic setting, \eg, in \cite{Nem-00-TopicsNonparametric,Tsy-03-OptimalRates,BuNo08SeqProcedures,ShSrZh-10-Sparsifiability}. It consists of discretizing $B_1(U)$ and looking at a random point in this discretization to study its approximation properties. We also use clipping to get a regret bound growing as $U$ instead of a naive~$U^2$. \\
  \ \\
  More precisely, we first use the fact that to be competitive against
  $B_1(U)$, it is sufficient to be competitive against its finite
  subset
  \[
  \tilde{B}_{U,m} \eqdef \left\{ \left(\frac{k_1 U}{m}, \ldots,
      \frac{k_d U}{m}\right) : (k_1, \ldots, k_d) \in \Z^d,
    \sum_{j=1}^d |k_j| \leq m \right\} \subset B_1(U)~,
  \]
  where $m \eqdef \lfloor \alpha \rfloor$ with $\displaystyle{\alpha \eqdef \frac{U X}{Y}  \sqrt{T (\ln 2)/ \ln\biggl(1+\frac{2dY}{\sqrt{T}UX}\biggr)}}$~. \\
  \ \\
  By Lemma~\ref{lem:Maurey-approx} in~\refapx{apx:additional}, and since $m>0$ (see below), we indeed have
  \begin{align}
  \inf_{\bu \in \tilde{B}_{U,m}} \sum_{t=1}^T (y_t - \bu \cdot \bx_t)^2 & \leq \inf_{\bu \in B_1(U)} \sum_{t=1}^T (y_t - \bu \cdot \bx_t)^2 + \frac{T U^2 X^2}{m} \nonumber \\
    & \leq \inf_{\bu \in B_1(U)} \sum_{t=1}^T (y_t - \bu \cdot
    \bx_t)^2 + \frac{2}{\sqrt{\ln 2}} \, U X Y
    \sqrt{T\ln\biggl(1+\frac{2dY}{\sqrt{T}UX}\biggr)}~, \label{eqn:upperbound-maurey-approx}
  \end{align}
  where (\ref{eqn:upperbound-maurey-approx}) follows from $m \eqdef
  \lfloor \alpha \rfloor \geq \alpha /2$ since $\alpha \geq 1$ (in particular, $m>0$ as stated above).

  To see why $\alpha \geq 1$, note that it suffices to show that $x \sqrt{\ln(1+x)} \leq 2d \sqrt{\ln 2}$ where we set $x \eqdef 2dY/(\sqrt{T}UX)$. But from the assumption $U \geq (Y/X) \sqrt{\ln(1+2d)/(T \ln 2)}$, we have $x \leq 2d\sqrt{\ln(2)/\ln(1+2d)} \eqdef y$, so that, by monotonicity, $x \sqrt{\ln(1+x)} \leq y \sqrt{\ln(1+y)} \leq y \sqrt{\ln(1+2d)} = 2d\sqrt{\ln 2}$. \\

  Therefore it only remains to exhibit an algorithm which is
  competitive against $\tilde{B}_{U,m}$ at an aggregation price of the
  same order as the last term
  in~(\ref{eqn:upperbound-maurey-approx}). This is the case for the
  standard exponentially weighted average forecaster applied to the
  clipped predictions
  \[
  \bigl[\bu \cdot \bx_t\bigr]_Y \eqdef \min\Bigl\{Y,\max\bigl\{-Y,\bu
  \cdot \bx_t\bigr\}\Bigr\}~, \quad \bu \in \tilde{B}_{U,m}~,
  \]
  and tuned with the inverse temperature parameter $\eta = 1/(8
  Y^2)$. More formally, this algorithm predicts at each time $t=1,
  \ldots, T$ as
  \[
  \hat{y}_t \eqdef \sum_{\bu \in \tilde{B}_{U,m}} p_t(\bu) \bigl[\bu
  \cdot \bx_t\bigr]_Y~,
  \]
  where $p_1(\bu) \eqdef 1/\bigl| \tilde{B}_{U,m} \bigr|$ (denoting by
  $\bigl| \tilde{B}_{U,m} \bigr|$ the cardinality of the set
  $\tilde{B}_{U,m}$), and where the weights $p_t(\bu)$ are defined for
  all $t=2,\ldots,T$ and $\bu \in \tilde{B}_{U,m}$ by
  \[
  p_t(\bu) \eqdef \frac{\exp\left(-\eta \sum_{s=1}^{t-1}
      \bigl(y_s-[\bu \cdot \bx_s]_Y\bigr)^2\right)}{\sum_{\bv \in
      \tilde{B}_{U,m}} \exp\left(-\eta \sum_{s=1}^{t-1} \bigl(y_s-[\bv
      \cdot \bx_s]_Y\bigr)^2\right)}~.
  \]
  By Lemma~\ref{lem:EWA-exp-concave} in~\refapx{sec:chapL1-lemmas}, the above forecaster
  tuned with $\eta = 1/(8 Y^2)$ satisfies
  \begin{align}
    \sum_{t=1}^T (y_t - \hat y_t)^2 - \inf_{\bu \in \tilde{B}_{U,m}} \sum_{t=1}^T (y_t - \bu \cdot \bx_t)^2 & \leq 8 Y^2 \ln \bigl| \tilde{B}_{U,m} \bigr| \nonumber \\
    & \leq 8 Y^2 \ln \left(\frac{\textrm{e}(2d+m)}{m}\right)^m \label{eqn:sauer-type-ineq} \\
    & = 8 Y^2 m \bigl(1+\ln(1+2d/m)\bigr) \leq 8 Y^2 \alpha \bigl(1+\ln(1+2d/\alpha)\bigr) \label{eqn:sauer-type-ineq2} \\
    & = 8 Y^2 \alpha + 8 Y^2 \alpha \ln\!\left(1+\frac{2dY}{\sqrt{T}UX}\sqrt{\frac{\ln\bigl(1+2dY/(\sqrt{T}UX)\bigr)}{\ln 2}}\,\right) \nonumber \\
    & \leq 8 Y^2 \alpha + 16 Y^2 \alpha \ln\!\left(1+\frac{2dY}{\sqrt{T}UX}\right) \label{eqn:sauer-type-ineq3} \\
    & \leq \left(\frac{8}{\sqrt{\ln 2}} + 16 \sqrt{\ln 2}\right)
    U X Y
    \sqrt{T\ln\biggl(1+\frac{2dY}{\sqrt{T}UX}\biggr)}~. \label{eqn:maurey-regret}
  \end{align}
  To get (\ref{eqn:sauer-type-ineq}) we used Lemma~\ref{lem:combinatorial} in \refapx{apx:additional}. Inequality~(\ref{eqn:sauer-type-ineq2}) follows by definition of $m \leq \alpha$ and the fact that $x \mapsto x \bigl(1+\ln(1+A/x)\bigr)$ is nondecreasing on $\R^*_+$ for all $A>0$. Inequality~(\ref{eqn:sauer-type-ineq3}) follows from the assumption $U \leq 2dY/(\sqrt{T}X)$ and the elementary inequality $\ln\bigl(1+x\sqrt{\ln(1+x)/\ln 2}\bigr) \leq 2 \ln(1+x)$ which holds for all $x \geq 1$ and was used, \eg, at the end of \cite[Theorem~2-a)]{BuNo08SeqProcedures}. Finally, elementary manipulations combined with the assumption that $2dY/(\sqrt{T}UX) \geq 1$ lead to~(\ref{eqn:maurey-regret}). \\

  Putting Eqs.\ (\ref{eqn:upperbound-maurey-approx})
  and~(\ref{eqn:maurey-regret}) together, the previous algorithm has a
  regret on $B_1(U)$ which is bounded from above by
  \[
  \left(\frac{10}{\sqrt{\ln 2}} + 16 \sqrt{\ln 2}\right) U X Y
  \sqrt{T\ln\biggl(1+\frac{2dY}{\sqrt{T}UX}\biggr)}~,
  \]
  which concludes the proof since $10/\sqrt{\ln 2} + 16 \sqrt{\ln 2} \leq 26$.
\end{proofref}

\subsection{Lower bound}
\label{sec:lowerbounds}

Corollary~\ref{cor:upperbound} gives an upper bound on the regret in terms of the quantities $d$, $Y$, and $\kappa \eqdef \sqrt{T}UX/(2dY)$.
We now show that for all $d \in \Zp$, $Y>0$, and $\kappa \geq \sqrt{\ln(1+2d)}/(2 d \sqrt{\ln 2})$, the upper bound can not be improved\footnote{For $T$ sufficiently large, we may overlook the case $\kappa < \sqrt{\ln(1+2d)}/(2 d \sqrt{\ln 2})$ or $\sqrt{T} < (Y/(UX)) \sqrt{\ln(1+2d)/\ln 2}$. Observe that in this case, the minimax regret is already of the order of $Y^2 \ln(1+d)$ (cf. Figure~\ref{fig:upperbound}).} up to logarithmic factors.

\vspace{0.3cm}
\begin{theorem}[Lower bound]
\label{thm:lowerbound}
For all $d \in \Zp$, $Y>0$, and $\kappa \geq \frac{\sqrt{\ln(1+2d)}}{2 d \sqrt{\ln 2}}$, there exist $T \geq 1$, $U>0$, and $X>0$ such that $\sqrt{T}UX/(2dY) = \kappa$ and
\begin{align*}
& \inf_F \sup_{\Norm{\bx_t}_\infty \leq X,\; \Abs{y_t}\leq Y} \Biggl\{ \sum_{t=1}^T (y_t - \hat y_t)^2 - \inf_{\Norm{\bu}_1 \leq U} \sum_{t=1}^T (y_t - \bu \cdot \bx_t)^2 \Biggr\} \\
& \quad \geq \left\{\begin{array}{ll}
\frac{c_1}{\ln\bigl(2+16 d^2\bigr)} d Y^2 \kappa \sqrt{\ln\left(1+1/\kappa\right)} & \textrm{if} \quad \frac{\sqrt{\ln(1+2d)}}{2 d \sqrt{\ln 2}} \leq \kappa \leq 1~, \\
\frac{c_2}{\ln\bigl(2+16 d^2\bigr)} d Y^2 & \textrm{if} \quad \kappa > 1~,
\end{array} \right.
\end{align*}
where $c_1,c_2>0$ are absolute constants. The infimum is taken over all forecasters $F$ and the
  supremum is taken over all sequences $(\bx_t,y_t)_{1 \leq t \leq T}
  \in (\R^d \times \R)^T$ such that $|y_1|, \ldots, |y_T| \leq Y$ and
  $\Norm{\bx_1}_{\infty}, \ldots, \Norm{\bx_T}_{\infty} \leq X$.
\end{theorem}

\vspace{0.2cm}
 The above lower bound extends those of \cite{CB99AnalysisGradientBased,KiWa97EGvsGD}, which hold for small $\kappa$ of the order of $1/d$. The proof is postponed to \refapx{sec:proof-thm2}. We perform a reduction to the stochastic batch setting---via the standard online to batch conversion---and employ a version of a lower bound of \cite{Tsy-03-OptimalRates}.

Note that in the proof of Theorem~\ref{thm:lowerbound}, we are free to choose the values of two parameters among $T$, $U$, and $X$, provided that $\sqrt{T}UX/(2dY) = \kappa$. This liberty is possible since the problem is now parametrized by $d$, $Y$, and $\kappa$ only (as shown in Corollary~\ref{cor:upperbound}, these three parameters are sufficient to express the regret bound of Theorem~\ref{thm:upperbound}, and they actually help to unify the upper bounds of the two regimes). A more ambitious lower bound would consist in proving that the upper bound of Theorem~\ref{thm:upperbound} cannot be substantially improved for any fixed value of $(d,Y,T,U,X)$. This question is left for future work.

%%%%%%%%%%%%%%%%%%%%%%%%%%%%%%%%%%%%%%%%%%%%%%%%%%%%%%%%%%%%%%%%%%%%%
\section{Adaptation to unknown $X$, $Y$ and $T$ via exponential weights}\label{sec:algo1}
%%%%%%%%%%%%%%%%%%%%%%%%%%%%%%%%%%%%%%%%%%%%%%%%%%%%%%%%%%%%%%%%%%%%%

Although the proof of Theorem~\ref{thm:upperbound} already gives an algorithm
that achieves the minimax regret,
the latter takes as inputs $U$, $X$, $Y$, and $T$, and it is inefficient in high dimensions.
In this section, we present a new method that achieves the minimax regret
both efficiently and without prior knowledge of $X$, $Y$, and $T$ provided that $U$ is known.
Adaptation to an unknown $U$ is considered in Section~\ref{sec:algo2}.
Our method consists of modifying an underlying efficient linear regression algorithm such
as the $\textrm{EG}^{\pm}$ algorithm \cite{KiWa97EGvsGD} or
the sequential ridge regression forecaster
  \cite{Vo01CompetitiveOnline,AzWa01RelativeLossBounds}.
Next, we show that automatically tuned variants of the $\textrm{EG}^{\pm}$ algorithm nearly achieve the minimax regret for the regime $d \geq \sqrt{T} UX/(2Y)$. A similar modification could be applied to the ridge regression forecaster --- with a total computational efficiency of the same order as that of the standard ridge algorithm --- to achieve a nearly optimal regret bound of order $d Y^2
  \ln \bigl(1+d\bigl(\frac{\sqrt{T} U X}{d Y}\bigr)^2\bigr)$ in the regime $d < \sqrt{T} UX/(2Y)$. The latter analysis is more technical and hence is omitted.

\subsection{An adaptive $\textrm{EG}^{\pm}$ algorithm for general convex and differentiable loss functions}
\label{sec:chapL1-adaptiveEGpm}

The second algorithm of the proof of Theorem~\ref{thm:upperbound} is
computationally inefficient because it aggregates approximately
$d^{\sqrt{T}}$ experts.
In contrast, the $\textrm{EG}^{\pm}$ algorithm has a
manageable computational complexity that is linear in~$d$ at each time $t$. Next we introduce a version of the $\textrm{EG}^{\pm}$ algorithm --- called the \emph{adaptive $\textrm{EG}^{\pm}$ algorithm} --- that does not
require prior knowledge of $X$, $Y$ and $T$ (as opposed to
the original $\textrm{EG}^{\pm}$ algorithm of \cite{KiWa97EGvsGD}). This version relies on the automatic tuning of \cite{CeMaSt07SecOrder}. We first present a generic version suited for general convex and differentiable loss functions. The application to the square loss and to other $\alpha$-losses will be dealt with in Sections~\ref{sec:adaptiveEG-square} and~\ref{sec:loss-Lip}.

The generic setting with arbitrary convex and differentiable loss functions corresponds to the online convex optimization setting \cite{Zin-03-GradientAscent,ShShSrSr-09-StochasticConvexOptimization} and unfolds as follows: at each time $t \geq 1$, the forecaster chooses a linear combination $\hat{\bu}_t \in \R^d$, then the environment chooses and reveals a convex and differentiable loss function $\ell_t:\R^d \to \R$, and the forecaster incurs the loss $\ell_t(\hat{\bu}_t)$. In online linear regression under the square loss, the loss functions are given by $\ell_t(\bu) = (y_t-\bu \cdot \bx_t)^2$.\\

\begin{figure}[ht]
  \begin{center}
    \bookbox{
      {\bf Parameter}: radius $U > 0$. \\[0.2cm]
      {\bf Initialization}: $\bp_1 = (p^+_{1,1},p^-_{1,1},\ldots,p^+_{d,1},p^-_{d,1}) \eqdef \bigl(1/(2d), \ldots, 1/(2d)\bigr) \in \R^{2d}$. \\[0.2cm]
      {\bf At each time round $t \geqslant 1$},
      \begin{enumerate}
			\vspace{-0.2cm}
      \item Output the linear combination $\displaystyle{\widehat{\bu}_{t} \eqdef U \sum_{j=1}^d \bigl(p^+_{j,t} - p^-_{j,t}\bigr) \, \be_j \in B_1(U)}$;
      \item Receive the loss function $\ell_t:\R^d \to \R$ and update the parameter $\eta_{t+1}$ according to \eqref{eqn:eta-variable};
      \item Update the weight vector $\bp_{t+1} = (p^+_{1,t+1},p^-_{1,t+1},\ldots,p^+_{d,t+1},p^-_{d,t+1}) \in \cX_{2d}$ defined for all $j=1,\ldots,d$ and $\gamma \in \{+,-\}$ by\footnote{For all $\gamma \in \{+,-\}$, by a slight abuse of notation, $\gamma U$ denotes $U$ or $-U$ if $\gamma=+$ or $\gamma=-$ respectively. }
      	\[
        p^{\gamma}_{j,t+1} \eqdef \frac{\displaystyle{\exp\!\left(- \eta_{t+1} \sum_{s=1}^{t} \gamma U \nabla_j \ell_s(\hat{\bu}_s)\right)}}{\displaystyle{\sum_{\substack{1 \leq k \leq d \\ \mu \in\{+,-\}}} \exp\!\left(- \eta_{t+1} \sum_{s=1}^{t} \mu U \nabla_k  \ell_s(\hat{\bu}_s)\right)}}~.
      	\]
      \end{enumerate}
    }
  \end{center}
  \vspace{-0.3cm}
  \caption{\label{fig:algo-adaptiveEG} The adaptive
    $\textrm{EG}^{\pm}$ algorithm for general convex and differentiable loss functions (see Proposition~\ref{prop:adaptiveEGpm}).}
    \vspace{0.5cm}
\end{figure}

The adaptive $\textrm{EG}^{\pm}$ algorithm for general convex and differentiable loss functions is defined in Figure~\ref{fig:algo-adaptiveEG}. We denote by $({\bf e}_j)_{1 \leq j \leq d}$ the canonical basis of $\R^d$, by $\nabla \ell_t(\bu)$ the gradient of $\ell_t$ at $\bu \in \R^d$, and by $\nabla_j\ell_t(\bu)$ the $j$-th component of this gradient. The adaptive $\textrm{EG}^{\pm}$ algorithm uses as a blackbox the exponentially weighted majority forecaster of \cite{CeMaSt07SecOrder} on $2d$ experts --- namely, the vertices $\pm U \be_j$ of $B_1(U)$ --- as in \cite{KiWa97EGvsGD}. It adapts to the unknown gradient amplitudes $\Arrowvert \nabla \ell_t \Arrowvert_{\infty}$ by the particular choice of $\eta_t$ due to \cite{CeMaSt07SecOrder} and defined for all $t \geq 2$ by
\begin{equation}
  \label{eqn:eta-variable}
  \eta_t = \min\left\{\frac{1}{\hat{E}_{t-1}}, \, C \sqrt{\frac{\ln(2d)}{V_{t-1}}} \, \right\}~,
\end{equation}
where $C \eqdef \sqrt{2 (\sqrt{2}-1)/(\rm{e}-2)}$ and where we set, for all $t=1,\ldots,T$,
				\begin{align*}
				z_{j,s}^+ & \eqdef U \nabla_j \ell_s(\hat{\bu}_s) \quad \textrm{and} \quad z_{j,s}^- \eqdef -U \nabla_j \ell_s(\hat{\bu}_s)~, \quad j=1,\ldots,d, \quad s=1,\ldots,t~,\\
				\hat{E}_t & \eqdef \inf_{k \in \Z} \left\{ 2^k : 2^k \geq \max_{1 \leq s \leq t} \max_{\substack{1 \leq j,k \leq d \\ \gamma,\mu \in \{+,-\}}} \, \bigl|z_{j,s}^{\gamma} - z_{k,s}^{\mu} \bigr| \right\}~, \\
				V_t & \eqdef \sum_{s=1}^t \sum_{\substack{1 \leq j \leq d\\ \gamma \in \{+,-\}}} p^{\gamma}_{j,s} \left(z_{j,s}^{\gamma} - \sum_{\substack{1 \leq k \leq d\\ \mu \in \{+,-\}}} p^{\mu}_{k,s} z_{k,s}^{\mu}\right)^2~.
				\end{align*}
Note that $\hat{E}_{t-1}$ approximates the range of the $z^{\gamma}_{j,s}$ up to time $t-1$, while $V_{t-1}$ is the corresponding cumulative variance of the forecaster.\\[1.5em]

\begin{proposition}[The adaptive
    $\textrm{EG}^{\pm}$ algorithm for general convex and differentiable loss functions]
\label{prop:adaptiveEGpm}
\ \\
Let $U>0$. Then, the adaptive $\textrm{EG}^{\pm}$ algorithm on $B_1(U)$ defined in Figure~\ref{fig:algo-adaptiveEG} satisfies, for all $T \geq 1$ and all sequences of convex and differentiable\footnote{Gradients can be replaced with subgradients if the loss functions $\ell_t:\R^d \to \R$ are convex but not differentiable.} loss functions $\ell_1,\ldots,\ell_T:\R^d \to \R$,
\begin{align*}
& \sum_{t=1}^T \ell_t(\hat{\bu}_t) - \min_{\norm[\bu]_1 \leq U} \sum_{t=1}^T \ell_t(\bu) \\
& \qquad \leq 4 U \sqrt{\left(\sum_{t=1}^T \norm[\nabla \ell_t(\hat{\bu}_t)]_{\infty}^2\right) \ln(2d)} + U \, \bigl(8 \ln(2d) + 12 \bigr) \max_{1 \leq t \leq T} \norm[\nabla \ell_t(\hat{\bu}_t)]_{\infty}~.
\end{align*}
In particular, the regret is bounded by $4U \bigl(\max_{1 \leq t \leq T} \norm[\nabla \ell_t(\hat{\bu}_t)]_{\infty}\bigr)\bigl( \sqrt{T \ln(2d)} + 2 \ln(2d) + 3 \bigr)$.
\end{proposition}

\begin{proof}
The proof follows straightforwardly from a linearization argument and from a regret bound of \cite{CeMaSt07SecOrder} applied to appropriately chosen loss vectors. Indeed, first note that by convexity and differentiability of $\ell_t:\R^d \to \R$ for all $t=1,\ldots,T$, we get that
\begin{align}
\sum_{t=1}^T \ell_t(\hat{\bu}_t) - \min_{\norm[\bu]_1 \leq U} \sum_{t=1}^T \ell_t(\bu) & = \max_{\norm[\bu]_1 \leq U} \sum_{t=1}^T \bigl(\ell_t(\hat{\bu}_t) - \ell_t(\bu) \bigr) \leq \max_{\norm[\bu]_1 \leq U} \sum_{t=1}^T \nabla \ell_t(\hat{\bu}_t) \cdot (\hat{\bu}_t - \bu) \nonumber \\
& = \max_{\substack{1 \leq j \leq d \\ \gamma \in \{+,-\}}} \sum_{t=1}^T \nabla \ell_t(\hat{\bu}_t) \cdot (\hat{\bu}_t - \gamma U \be_j) \label{eqn:introM-EG-eta-t-1} \\
& = \sum_{t=1}^T \sum_{\substack{1 \leq j \leq d \\ \gamma \in \{+,-\}}} p^{\gamma}_{j,t} \, \gamma U \nabla_j \ell_t(\hat \bu_t) - \min_{\substack{1 \leq j \leq d \\ \gamma \in \{+,-\}}} \sum_{t=1}^T \gamma U \nabla_j \ell_t(\hat \bu_t)~,\label{eqn:introM-EG-eta-t-2}
\end{align}
where \eqref{eqn:introM-EG-eta-t-1} follows by linearity of $\bu \mapsto \sum_{t=1}^T \nabla \ell_t(\hat{\bu}_t) \cdot (\hat{\bu}_t - \bu)$ on the polytope $B_1(U)$, and where \eqref{eqn:introM-EG-eta-t-2} follows from the particular choice of $\hat{\bu}_t$ in Figure~\ref{fig:algo-adaptiveEG}.

To conclude the proof, note that our choices of the weight vectors $\bp_t \in \cX_{2 d}$ in Figure~\ref{fig:algo-adaptiveEG} and of the time-varying parameter $\eta_t$ in \eqref{eqn:eta-variable} correspond to the exponentially weighted average forecaster of~\cite[Section~4.2]{CeMaSt07SecOrder} when it is applied to the loss vectors $\bigl(U \nabla_j \ell_t(\widehat{\bu}_t), - U \nabla_j \ell_t(\widehat{\bu}_t)\bigr)_{1 \leq j \leq d} \in \R^{2 d}$, $t=1,\ldots,T$. Since at time~$t$ the coordinates of the last loss vector lie in an interval of length $E_t \leq 2 U \norm[\nabla \ell_t(\hat{\bu}_t)]_{\infty}$, we get from \cite[Corollary~1]{CeMaSt07SecOrder} that
\begin{align*}
& \sum_{t=1}^T \sum_{\substack{1 \leq j \leq d \\ \gamma \in \{\pm
        1\}}} p^{\gamma}_{j,t} \, \gamma U \nabla_j \ell_t(\hat \bu_t) - \min_{\substack{1 \leq j \leq d \\ \gamma \in \{\pm
        1\}}} \sum_{t=1}^T \gamma U \nabla_j \ell_t(\hat \bu_t) \nonumber \\
& \qquad \leq  4 U \sqrt{\left(\sum_{t=1}^T \norm[\nabla \ell_t(\hat{\bu}_t)]_{\infty}^2\right) \ln(2d)} + U \, \bigl(8 \ln(2d) + 12 \bigr) \max_{1 \leq t \leq T} \norm[\nabla \ell_t(\hat{\bu}_t)]_{\infty}~.
\end{align*}
Substituting the last upper bound in \eqref{eqn:introM-EG-eta-t-2} concludes the proof.
\end{proof}

\subsection{Application to the square loss}
\label{sec:adaptiveEG-square}

In the particular case of the square loss $\ell_t(\bu) = (y_t - \bu \cdot \bx_t)^2$, the gradients are given by $\nabla \ell_t(\bu) = -2 (y_t - \bu \cdot \bx_t) \, \bx_t$ for all $\bu \in \R^d$. Applying Proposition~\ref{prop:adaptiveEGpm}, we get the following regret bound for the adaptive $\textrm{EG}^{\pm}$ algorithm.

\vspace{0.2cm}
\begin{corollary}[The adaptive $\textrm{EG}^{\pm}$ algorithm under the square loss]
\label{cor:introM-EGpm-square}
\ \\
Let $U>0$. Consider the online linear regression setting defined in the introduction. Then, the adaptive $\textrm{EG}^{\pm}$ algorithm (see Figure~\ref{fig:algo-adaptiveEG}) tuned with $U$ and applied to the loss functions $\ell_t: \bu \mapsto (y_t - \bu \cdot \bx_t)^2$ satisfies, for all individual sequences $(\bx_1,y_1),\ldots,(\bx_T,y_T) \in \R^d \times \R$,
\begin{align*}
& \sum_{t=1}^T (y_t - \hat{\bu}_t \cdot \bx_t)^2 - \min_{\norm[\bu]_1 \leq U} \sum_{t=1}^T (y_t - \bu \cdot \bx_t)^2 \\
& \qquad \leq 8 U X \sqrt{\left(\min_{\norm[\bu]_1 \leq U} \sum_{t=1}^T (y_t - \bu \cdot \bx_t)^2\right) \ln(2d)} + \bigl(137 \ln(2d) + 24\bigr) \, \bigl(UXY + U^2 X^2\bigr) \\
& \qquad \leq 8 U X Y \sqrt{T \ln(2d)} + \bigl(137 \ln(2d) + 24\bigr) \, \bigl(UXY + U^2 X^2\bigr)~,
\end{align*}
where the quantities $X \eqdef \max_{1 \leq t \leq T} \norm[\bx_t]_{\infty}$ and $Y \eqdef \max_{1 \leq t \leq T} |y_t|$ are unknown to the forecaster.
\end{corollary}

\vspace{0.2cm}
Using the terminology of \cite{cesa-bianchi06prediction,CeMaSt07SecOrder}, the first bound of Corollary~\ref{cor:introM-EGpm-square} is an \emph{improvement for small losses}: it yields a small regret when the optimal cumulative loss $\min_{\norm[\bu]_1 \leq U} \sum_{t=1}^T (y_t - \bu \cdot \bx_t)^2$ is small. As for the second regret bound, it indicates that the adaptive $\textrm{EG}^{\pm}$ algorithm achieves approximately the regret bound of Theorem~\ref{thm:upperbound} in the regime $\kappa \leq 1$, \ie, $d \geq \sqrt{T} UX/(2Y)$. In this regime, our algorithm thus has a manageable computational complexity (linear in $d$ at each time $t$) and it is adaptive in $X$, $Y$, and $T$.\\

In particular, the above regret bound is similar\footnote{By Theorem~5.11 of \cite{KiWa97EGvsGD}, the original $\textrm{EG}^{\pm}$ algorithm satisfies the regret bound $2 U X \sqrt{2 B \ln(2d)} + 2 U^2 X^2 \ln(2d)$, where $B$ is an upper bound on $\min_{\norm[\bu]_1 \leq U} \sum_{t=1}^T (y_t - \bu \cdot \bx_t)^2$ (in particular, $B \leq T Y^2$). Note that our main regret term is larger by a multiplicative factor of~$2 \sqrt{2}$. However, contrary to \cite{KiWa97EGvsGD}, our algorithm does not require the prior knowledge of $X$ and $B$ --- or, alternatively, $X$, $Y$, and $T$.} to that of the original $\textrm{EG}^{\pm}$ algorithm \cite[Theorem~5.11]{KiWa97EGvsGD}, but it is obtained without prior knowledge of $X$, $Y$, and $T$. Note also that this bound is similar to that of the self-confident $p$-norm algorithm of~\cite{AuCeGe02Adaptive} with $p=2 \ln d$ (see Section~\ref{sec:intro-contributions}). The fact that we were able to get similar adaptivity and efficiency properties via exponential weighting corroborates the similarity that was already observed in a non-adaptive context between the original $\textrm{EG}^{\pm}$ algorithm and the $p$-norm algorithm (in the limit $p \to +\infty$ with an appropriate initial weight vector, or for $p$ of the order of $\ln d$ with a zero initial weight vector, {cf.} \cite{Ge-03-pNorm}).

\vspace{0.3cm}
\begin{proofref}{Corollary~\ref{cor:introM-EGpm-square}}
We apply Proposition~\ref{prop:adaptiveEGpm} with the square loss $\ell_t(\bu) = (y_t - \bu \cdot \bx_t)^2$. It yields
\begin{align}
& \sum_{t=1}^T \ell_t(\hat{\bu}_t) - \min_{\norm[\bu]_1 \leq U} \sum_{t=1}^T \ell_t(\bu) \nonumber \\
& \qquad \leq 4 U \sqrt{\left(\sum_{t=1}^T \norm[\nabla \ell_t(\hat{\bu}_t)]_{\infty}^2\right) \ln(2d)} + U \, \bigl(8 \ln(2d) + 12 \bigr) \max_{1 \leq t \leq T} \norm[\nabla \ell_t(\hat{\bu}_t)]_{\infty}~. \label{eqn:adaptiveEG-square-1}
\end{align}
Using the equality $\nabla \ell_t(\bu) = -2 (y_t - \bu \cdot \bx_t) \, \bx_t$ for all $\bu \in \R^d$, we get that, on the one hand, by the upper bound $\norm[\bx_t]_{\infty} \leq X$,
\begin{equation}
\label{introM-subquadratic}
\norm[\nabla \ell_t(\hat{\bu}_t)]_{\infty}^2 \leq 4 X^2 \ell_t(\hat{\bu}_t)~,
\end{equation}
and, on the other hand, $\max_{1 \leq t \leq T} \norm[\nabla  \ell_t(\hat{\bu}_t)]_{\infty} \leq  2(Y+UX)X$ (indeed, by H\"{o}lder's inequality, $ \bigl|\hat{\bu}_t \cdot \bx_t\bigr| \leq \norm[\hat{\bu}_t]_1 \norm[\bx_t]_{\infty} \leq U X$). Substituting the last two inequalities in~\eqref{eqn:adaptiveEG-square-1}, setting $\hat{L}_T \eqdef \sum_{t=1}^T \ell_t(\hat{\bu}_t)$ as well as $L_T^* \eqdef \min_{\norm[\bu]_1 \leq U} \sum_{t=1}^T \ell_t(\bu)$, we get that
\begin{align*}
\hat{L}_T \leq L_T^* + 8 U X \sqrt{\hat{L}_T \ln(2d)} + \underbrace{\bigl(16 \ln(2d) + 24 \bigr) \bigl(U X Y + U^2 X^2\bigr)}_{\eqdef \, C}~.
\end{align*}
Solving for $\hat{L}_T$ via Lemma~\ref{lem:introM-solvingRegret} in~\refapx{sec:chapL1-lemmas}, we get that
\begin{align*}
\hat{L}_T & \leq L_T^* + C + \left( 8 U X \sqrt{\ln(2d)} \right) \sqrt{L_T^* + C} + \left( 8 U X \sqrt{\ln(2d)} \right)^2 \\
& \leq L_T^* + 8 U X \sqrt{L_T^* \ln(2d)} + 8 U X \sqrt{C \ln(2d)} + 64 U^2 X^2 \ln(2d) + C~.
\end{align*}
Using that
\begin{align*}
UX \sqrt{C \ln(2d)} & = UX \ln(2d) \sqrt{\bigr(16+24/\ln(2d)\bigr) \bigl(U X Y + U^2 X^2\bigr)} \\
& \leq \sqrt{U^2X^2+UXY} \, \ln(2d) \, \sqrt{\bigr(16+24/\ln(2)\bigr) \bigl(U X Y + U^2 X^2\bigr)} \\
& = \sqrt{16+24/\ln(2)} \, \bigl(U X Y + U^2 X^2\bigr) \ln(2d)
\end{align*}
and performing some simple upper bounds concludes the proof of the first regret bound. The second one follows immediately by noting that $\min_{\norm[\bu]_1 \leq U} \sum_{t=1}^T (y_t - \bu \cdot \bx_t)^2 \leq \sum_{t=1}^T y_t^2 \leq T Y^2$ (since ${\bf 0} \in B_1(U)$).
\end{proofref}

\subsection{A refinement via Lipschitzification of the loss function}
\label{sec:loss-Lip}

In Corollary~\ref{cor:introM-EGpm-square} we used the adaptive $\textrm{EG}^{\pm}$ algorithm in conjunction with the square loss functions ${\ell_t: \bu \mapsto (y_t - \bu \cdot \bx_t)^2}$. In this section we use yet another instance of the adaptive $\textrm{EG}^{\pm}$ algorithm applied to a modification $\tilde{\ell}_t:\R^d \to \R$ of the square loss (or the $\alpha$-loss, see below) which is Lipschitz continuous with respect to~$\Norm{\cdot}_1$. This leads to slightly refined regret bounds; see Theorem~\ref{thm:ell1-EG-regret} below and Corollaries~\ref{cor:adaptiveEG-lip-square} and~\ref{cor:adaptiveEG-lip-higherCurvature} thereafter.

We first present the Lipschtizification technique; its use with the adaptive $\textrm{EG}^{\pm}$ algorithm is to be addressed in a few paragraphs. Since our analysis is generic enough to handle both the square loss and other loss functions with higher curvature, we consider below a slightly more general setting than online linear regression \emph{stricto sensu}. Namely, we fix a real number $\alpha \geq 2$ and assume that the predictions $\hat{y}_t$ of the forecaster and the base linear predictions $\bu \cdot \bx_t$ are scored with the $\alpha$-loss, \ie, with the loss functions $x \mapsto |y_t - x |^{\alpha}$ for all $t \geq 1$. The particular case of the square loss ($\alpha=2$) is considered in Corollary~\ref{cor:adaptiveEG-lip-square} below, while loss functions with higher curvature ($\alpha > 2$) are addressed in Corollary~\ref{cor:adaptiveEG-lip-higherCurvature}.

The Lipschitzification proceeds as follows. At each time
$t \geq 1$, we set
\[
B_t \eqdef \left(2^{\lceil \log_2 (\max_{1 \leq s \leq t-1} |y_s|^{\alpha}) \rceil}\right)^{1/\alpha}~,
\]
where $\lceil x \rceil \eqdef \min \{k \in \Z: k \geq x\}$ for all $x \in \R$. Note that $\max_{1 \leq s \leq t-1} |y_s| \leq  B_t \leq 2^{1/\alpha} \max_{1 \leq s \leq t-1} |y_s|$.
The modified (or \emph{Lipschitzified}) loss function $\widetilde{\ell}_t:\R^d \to \R$
is constructed as follows:
\begin{itemize}
\item if $|y_t| > B_t$, then
  \[
  \widetilde{\ell}_t(\bu) \eqdef 0 \quad \textrm{for all } \bu \in \R^d~;
  \]
\item if $|y_t| \leq B_t$, then $\widetilde{\ell}_t$ is the convex function that coincides with the loss function $\bu \mapsto |y_t - \bu \cdot \bx_t|^{\alpha}$ when $\big|\bu \cdot \bx_t\big| \leq B_t$ and is
  linear elsewhere.  An example of such function is shown in Figure~\ref{fig:semi-linearization} in the case where $\alpha=2$. It can be formally defined as
\begin{align*}
\widetilde{\ell}_t(\bu) \eqdef \left\{\begin{array}{ll}
	\big|y_t-\bu \cdot \bx_t\big|^{\alpha} &\quad \mbox{if} \quad \big|\bu \cdot \bx_t\big| \leq B_t ,\\
	\big|y_t-B_t\big|^{\alpha} + \alpha \big|y_t-B_t\big|^{\alpha-1} (\bu \cdot \bx_t - B_t) &\quad\mbox{if} \quad \bu \cdot \bx_t > B_t, \\
	\big|y_t+B_t\big|^{\alpha} - \alpha \big|y_t+B_t\big|^{\alpha-1} (\bu \cdot \bx_t + B_t) &\quad
	\mbox{if} \quad \bu \cdot \bx_t < - B_t.
  \end{array} \right.
\end{align*}
\end{itemize}

Observe that in both cases $|y_t| > B_t$ and $|y_t| \leq B_t$, the function $\widetilde{\ell}_t$ is continuously differentiable. By construction it is also Lipschitz continuous with respect to
$\Norm{\cdot}_1$ with an easy-to-control Lipschitz constant (see \refapx{sec:chapL1-proof-EG-regret}). Another key property that we can glean from Figure~\ref{fig:semi-linearization} is
that, when $|y_t| \leq B_t$, the modified loss function $\widetilde{\ell}_t:\R^d \to \R$ lies in between the $\alpha$-loss function $\bu \mapsto |y_t - \bu \cdot \bx_t|^{\alpha}$ and its clipped version:
\begin{equation}
  \label{eqn:Lip-properties}
  \forall \bu \in \R^d, \quad \bigl|y_t - [\bu \cdot \bx_t]_{B_t} \bigr|^{\alpha} \leq \tilde{\ell}_t(\bu) \leq \bigl|y_t - \bu \cdot \bx_t \bigr|^{\alpha}~,
\end{equation}
where the clipping operator $[\cdot]_B$ is defined by $[x]_B \eqdef \min\bigl\{B,\max\{-B,x\}\bigr\}$ for all $x \in \R$ and all $B>0$.\\

\begin{figure}[htp]
  \vspace{-0.5cm}
  \begin{center}
   	\includegraphics[scale=0.8]{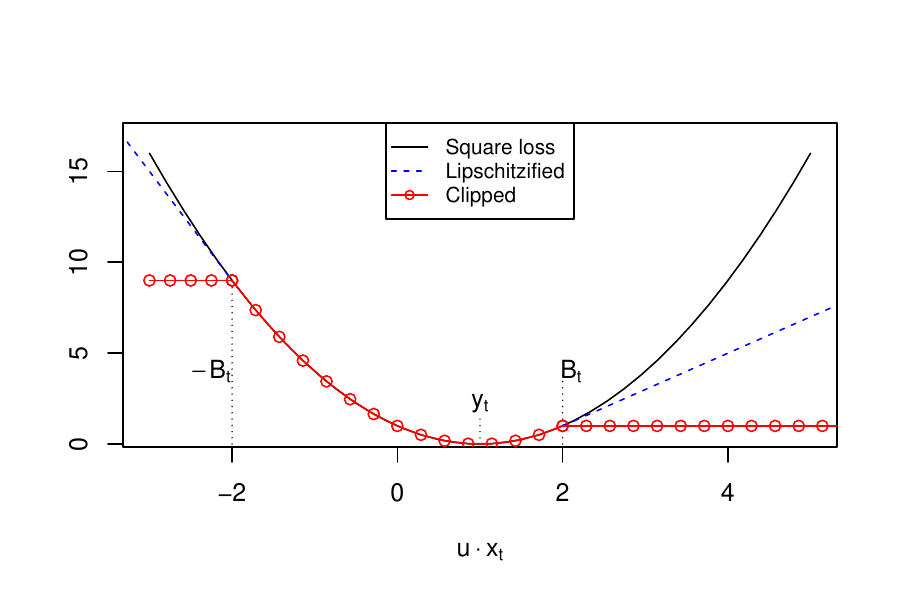}
  \end{center}
  \vspace{-0.8cm}
  \caption{\label{fig:semi-linearization} Example with the square loss ($\alpha=2$) when $|y_t| \leq B_t$. The square loss
    $(y_t - \bu \cdot \bx_t)^2$, its clipped version $\bigr(y_t - [\bu
    \cdot \bx_t]_{B_t}\bigr)^2$ and its Lipschitzified version
    $\tilde{\ell}_t(\bu)$ are plotted as a function of $\bu \cdot
    \bx_t$.}
\end{figure}

Next we illustrate the Lipschitzification technique introduced above: we apply the adaptive $\textrm{EG}^{\pm}$ algorithm to the Lipschitzified loss functions $\tilde{\ell}_t$. The resulting algorithm is called  the \emph{Lipschitzifying Exponentiated Gradient} (LEG) algorithm and is formally defined in Figure~\ref{fig:algo-LEG}. Recall that $({\bf e}_j)_{1 \leq j \leq d}$ denotes the canonical basis of $\R^d$ and that $\nabla_j$ denotes the $j$-th component of the gradient.

\begin{figure}[ht]
  \begin{center}
    \bookbox{
      {\bf Parameter}: radius $U > 0$. \\[0.2cm]
      {\bf Initialization}: $B_1 \eqdef 0$, $\bp_1 = (p^+_{1,1},p^-_{1,1},\ldots,p^+_{d,1},p^-_{d,1}) \eqdef \bigl(1/(2d), \ldots, 1/(2d)\bigr) \in \R^{2d}$. \\[0.2cm]
      {\bf At each time round $t \geqslant 1$},
      \begin{enumerate}
      \item Compute the linear combination $\displaystyle{\widehat{\bu}_{t} \eqdef U \sum_{j=1}^d \bigl(p^+_{j,t} - p^-_{j,t}\bigr) \, \be_j \in B_1(U)}$;
      \vspace{-0.2cm}
      \item Get $\bx_t \in \R^d$ and output the clipped prediction
        $\widehat{y}_t \eqdef \bigl[\widehat{\bu}_t \cdot
        \bx_t\bigr]_{B_t}$;
      \item Get $y_t \in \R$ and define the modified loss function
        $\widetilde{\ell}_t:\R^d \to \R$ as above;
      \item Update the parameter $\eta_{t+1}$ according to \eqref{eqn:eta-variable};
      \item Update the weight vector $\bp_{t+1} = (p^+_{1,t+1},p^-_{1,t+1},\ldots,p^+_{d,t+1},p^-_{d,t+1}) \in \cX_{2d}$ defined for all $j=1,\ldots,d$ and $\gamma \in \{+,-\}$ by\footnote{For all $\gamma \in \{+,-\}$, by a slight abuse of notation, $\gamma U$ denotes $U$ or $-U$ if $\gamma=+$ or $\gamma=-$ respectively. }
      	\[
        p^{\gamma}_{j,t+1} \eqdef \frac{\displaystyle{\exp\!\left(- \eta_{t+1} \sum_{s=1}^{t} \gamma U \nabla_j \tilde{\ell}_s(\hat{\bu}_s)\right)}}{\displaystyle{\sum_{\substack{1 \leq k \leq d \\ \mu \in\{+,-\}}} \exp\!\left(- \eta_{t+1} \sum_{s=1}^{t} \mu U \nabla_k  \tilde{\ell}_s(\hat{\bu}_s)\right)}}~.
      	\]
      \item Update the threshold $\quad B_{t+1} \eqdef \left(2^{\lceil
            \log_2( \max_{1 \leq s \leq t} |y_s|^{\alpha})
            \rceil}\right)^{1/\alpha}$~.
      \end{enumerate}
    }
  \end{center}
  \vspace{-0.5cm}
  \caption{\label{fig:algo-LEG} The Lipschitzifying
    Exponentiated Gradient (LEG) algorithm.}
\end{figure}

We point out that this technique is not specific to the pair of dual norms $(\norm[\cdot]_1,\norm[\cdot]_{\infty})$ and to the $\textrm{EG}^{\pm}$ algorithm; it could be used with other pairs $(\norm[\cdot]_q,\norm[\cdot]_p)$ (with $1/p+1/q=1$) and other gradient-based algorithms, such as the $p$-norm algorithm \cite{Ge-03-pNorm,AuCeGe02Adaptive} and its regularized variants (SMIDAS and COMID) \cite{ShaTe-09-StochasticL1Regularization,DuShaSiTe-10-CompostiveObjectiveMD}.

The next theorem bounds the cumulative $\alpha$-loss of the LEG algorithm. The proof is postponed to \refapx{sec:chapL1-proof-EG-regret}. It follows from the bound on the adaptive $\textrm{EG}^{\pm}$ algorithm for general convex and differentiable loss functions that we derived in Proposition~\ref{prop:adaptiveEGpm} (Section~\ref{sec:chapL1-adaptiveEGpm}). See Corollaries~\ref{cor:adaptiveEG-lip-square} and~\ref{cor:adaptiveEG-lip-higherCurvature} below for regret bounds in the particular cases of the square loss ($\alpha=2$) or of losses with higher curvature ($\alpha>2$).

\vspace{0.3cm}
\begin{theorem}
  \label{thm:ell1-EG-regret}
  Assume that the predictions are scored with the $\alpha$-loss $x \mapsto |y_t - x|^{\alpha}$, where $\alpha \geq 2$ is a real number. Let $U > 0$. Then, the
  LEG algorithm defined in Figure~\ref{fig:algo-LEG} and tuned with~$U$
  satisfies, for all $T \geq 1$ and all individual sequences
  $(\bx_1, y_1),$ $\ldots,(\bx_T, y_T) \in \R^d \times \R$,
  \begin{align*}
    \sum_{t=1}^T \abs[y_t - \hat y_t]^{\alpha} & \leq \inf_{\Norm{\bu}_1 \leq U} \sum_{t=1}^T \tilde{\ell}_t(\bu) + a_{\alpha} U X Y^{\alpha/2-1} \sqrt{\left(\inf_{\Norm{\bu}_1 \leq U} \sum_{t=1}^T \tilde{\ell}_t(\bu)\right) \ln(2d)} \\[0.2cm]
    & \qquad  + \Bigl(a'_{\alpha} \ln(2d) + 12 b_{\alpha}\Bigr) \, U X Y^{\alpha-1} + a''_{\alpha} \ln(2d) \, U^2 X^2 Y^{\alpha-2} + a'''_{\alpha} \, Y^{\alpha}~,
  \end{align*}
where the Lipschitzified loss functions $\tilde{\ell}_t$ are defined above, where the quantities $X \eqdef \max_{1 \leq t \leq T} \norm[\bx_t]_{\infty}$ and $Y \eqdef \max_{1 \leq t \leq T} |y_t|$ are unknown to the forecaster, and where, setting $a_{\alpha} \eqdef 4 \alpha \, \bigl(1+2^{1/\alpha}\bigr)^{\alpha/2-1}$ and $b_{\alpha} \eqdef \alpha \, \bigl(1+2^{1/\alpha}\bigr)^{\alpha-1}$, the constants $a'_{\alpha}, a''_{\alpha}, a'''_{\alpha} > 0$ are defined by
\begin{align*}
\left\{ \begin{array}{l}
	a'_{\alpha} \eqdef a_{\alpha} \left( \sqrt{b_{\alpha} \bigl(4+6/\ln 2 \bigr)} + 2 \bigl(1+2^{-1/\alpha}\bigr)^{\alpha/2} / \sqrt{\ln 2} \right) + 8 b_{\alpha}\\
a''_{\alpha} \eqdef a_{\alpha} \left( \sqrt{b_{\alpha} \bigl(4+6/\ln 2 \bigr)} + a_{\alpha} \right)\\
a'''_{\alpha} \eqdef  4 \bigl(1+2^{-1/\alpha}\bigr)^{\alpha}~.
\end{array}\right.
\end{align*}
\end{theorem}

\vspace{0.3cm}
\begin{corollary}[Application to the square loss]
\label{cor:adaptiveEG-lip-square}
  Consider the online linear regression setting under the square loss (\ie, $\alpha=2$). Let $U > 0$. Then, the
  LEG algorithm defined in Figure~\ref{fig:algo-LEG} and tuned with~$U$
  satisfies, for all $T \geq 1$ and all individual sequences
  $(\bx_1, y_1),$ $\ldots,(\bx_T, y_T) \in \R^d \times \R$,
  \begin{align*}
    \sum_{t=1}^T (y_t - \hat y_t)^2 & \leq \inf_{\Norm{\bu}_1 \leq U} \sum_{t=1}^T \tilde{\ell}_t(\bu) + 8 U X \sqrt{\left(\inf_{\Norm{\bu}_1 \leq U} \sum_{t=1}^T \tilde{\ell}_t(\bu)\right) \ln(2d)} \\
    & \qquad  + \bigl(134 \ln(2d) + 58\bigr) \, \bigl(UXY + U^2 X^2\bigr) +  12 Y^2~,
  \end{align*}
where the Lipschitzified loss functions $\tilde{\ell}_t$ are defined above and where the quantities $X \eqdef \max_{1 \leq t \leq T} \norm[\bx_t]_{\infty}$ and $Y \eqdef \max_{1 \leq t \leq T} |y_t|$ are unknown to the forecaster.
\end{corollary}

\vspace{0.4cm}
Note that, in the case of the square loss, the first two terms of the bound of Corollary~\ref{cor:adaptiveEG-lip-square} slightly improve on those obtained without Lipschitzification ({cf.}\ Corollary~\ref{cor:introM-EGpm-square}) since we always have
\[
\inf_{\Norm{\bu}_1 \leq U} \sum_{t=1}^T \tilde{\ell}_t(\bu) \leq \inf_{\norm[\bu]_1 \leq U} \sum_{t=1}^T (y_t - \bu \cdot \bx_t)^2~,
\]
where we used the key property $\tilde{\ell}_t(\bu) \leq (y_t - \bu \cdot \bx_t)^2$ that holds for all $\bu \in \R^d$ and all $t=1,\ldots,T$ (by \eqref{eqn:Lip-properties} if $|y_t| \leq B_t$, obvious otherwise). In particular, the LEG algorithm is adaptive in $X$, $Y$, and $T$; it achieves approximately --- and efficiently~--- the regret bound of Theorem~\ref{thm:upperbound} in the regime $\kappa \leq 1$, \ie, $d \geq \sqrt{T} UX/(2Y)$.
\\
%
%Though there are situations where the ratio between the left- and right-hand sides of \eqref{eqn:chapL1-EGpm-gain} is smaller than a constant $c < 1$, such situations usually occur in hard-to-learn problems, \ie, when both quantities are large. Hence the for the square loss, the benefits of Lipschzification appear to be of minor interest.

In the case of $\alpha$-losses with a higher curvature than that of the square loss ($\alpha > 2$), the improvement is more substantial as indicated after the following corollary.

\vspace{0.3cm}
\begin{corollary}[Application to $\alpha$-losses with $\alpha>2$]
\label{cor:adaptiveEG-lip-higherCurvature}
 Assume that the predictions are scored with the $\alpha$-loss $x \mapsto |y_t - x|^{\alpha}$, where $\alpha > 2$. Then, the regret of
the LEG algorithm on $B_1(U)$ is at most of the order of
\[
U X Y^{\alpha-1} \sqrt{T \ln(2d)} + \Bigl( U X Y^{\alpha-1} + U^2 X^2 Y^{\alpha-2} \Bigr) \ln(2d) + Y^{\alpha}~,
\]
where $X \eqdef \max_{1 \leq t \leq T} \norm[\bx_t]_{\infty}$ and $Y \eqdef \max_{1 \leq t \leq T} |y_t|$ are unknown to the forecaster. The above regret bound improves on the bound we would have obtained via a similar analysis for the adaptive $\textrm{EG}^{\pm}$ algorithm applied to the original losses $\ell_t(\bu) = | y_t - \bu \cdot \bx_t|^{\alpha}$ (without Lipschitzification), namely, a bound of the order of
\[
U X (Y+UX)^{\alpha/2-1} \, Y^{\alpha/2} \sqrt{T \ln(2d)} + \Bigl( U X (Y+UX)^{\alpha-1} \, + U^2 X^2 (Y+UX)^{\alpha-2} \Bigr) \ln(2d)~.
\]
\end{corollary}

\vspace{0.3cm}
The main difference between the two regret bounds above lies in the dependence in $U$: our main regret term scales as $U X Y^{\alpha-1}$ while the one obtained without Lipschitzification scales as $U X (Y+UX)^{\alpha/2-1} \, Y^{\alpha/2}$. The first term grows linearly in $U$ while the second one grows as $U^{\alpha/2}$, hence a clear improvement for $\alpha > 2$. The last property stems from the fact that, thanks to Lipschitzification, the gradients $\Norm{\nabla \tilde{\ell}_t}_{\infty}$ are bounded as $U \to +\infty$ ({cf.\ \eqref{eqn:Lip-constant} in \refapx{sec:chapL1-proof-EG-regret}).

\vspace{0.3cm}
\begin{remark}[Another benefit of Lipschitzification]
\label{rem:chapL1-simpler-analysis}
\ \\
Another benefit of Lipschitzification is that all online convex optimization regret bounds expressed in terms of the maximal dual norm of the gradients --- \ie, $\max_{1 \leq t \leq T} \Arrowvert \nabla \tilde{\ell}_t \Arrowvert_{\infty}$ in our case --- can be used fruitfully with the Lipschitzified loss functions $\tilde{\ell}_t$. For instance, in the case of the square loss, using the very last bound of Proposition~\ref{prop:adaptiveEGpm}, we get that
  \begin{align*}
    \sum_{t=1}^T (y_t - \hat y_t)^2 - \inf_{\Norm{\bu}_1 \leq U} \sum_{t=1}^T (y_t - \bu \cdot \bx_t)^2 & \leq  c_1 U X Y \left(\sqrt{T \ln (2 d)} + 8 \ln (2
      d)\right) + c_2 Y^2~,
  \end{align*}
where $c_1 \eqdef 8 \bigl(\sqrt{2}+1\bigr)$ and $c_2 \eqdef 4\left(1+1/\sqrt{2}\right)^2$. The bound is no longer an improvement for small losses (as compared to Corollary~\ref{cor:introM-EGpm-square}), but it does not require to solve any quadratic inequality. The corresponding simple proof is postponed to the end of \refapx{sec:chapL1-proof-EG-regret}.
\end{remark}

%%%%%%%%%%%%%%%%%%%%%%%%%%%%%%%%%%%%%%%%%%%%%%%%%%%%%%%%%%%%%%%%%%%%%
\section{Adaptation to unknown $U$}\label{sec:algo2}
%%%%%%%%%%%%%%%%%%%%%%%%%%%%%%%%%%%%%%%%%%%%%%%%%%%%%%%%%%%%%%%%%%%%%

In the previous section, the forecaster is given a radius
$U>0$ and asked to ensure a low worst-case regret on the
$\ell^1$-ball $B_1(U)$. In this section, $U$ is no longer given: the
forecaster is asked to be competitive against all balls
$B_1(U)$, for $U>0$.
Namely, its worst-case regret on each $B_1(U)$ should be almost as
good as if $U$ were known beforehand. 
For simplicity, we assume that $X$, $Y$, and $T$ are known: we explain in Section~\ref{sec:dis} how to simultaneously adapt to all parameters. Note that from now on, we consider again the main framework of this paper, \ie, online linear regression under the square loss ({cf.}\ Section~\ref{sec:set}).\\

\begin{figure}[h]
  \centering \framebox{
    \begin{minipage}{0.75\textwidth}
      \textbf{Parameters}: $X, Y, \eta > 0$, $T \geq 1$, and $c>0$ (a constant).\\
      \textbf{Initialization}: $R = \lceil \log_2 (2T/c) \rceil_+$, $\bw_1 = \Bigl(\frac{1}{R+1},\cdots,\frac{1}{R+1}\Bigr) \in \R^{R+1}$.\\
      For time steps $t=1,\ldots,T$:
      \begin{enumerate}
      \item For experts $r = 0,\ldots,R$:
        \begin{itemize}
        \item Run the sub-algorithm $\mathcal A(U_r)$ on the ball
          $B_1(U_r)$ and obtain the prediction $\hat y_t^{(r)}$.
        \end{itemize}
      \item Output the prediction $\hat y_t = \sum_{r=0}^R
        \frac{w_t^{(r)}}{\sum_{r'=0}^R w_{t}^{(r')}} \big[ \hat y_t^{(r)}
        \big]_Y$.
      \item Update $w_{t+1}^{(r)} = w_t^{(r)} \exp\left(-\eta \big(y_t
          - \big[\hat y_t^{(r)} \big]_Y \big)^2\right)$ for $r=0,\ldots,R$.
      \end{enumerate}
    \end{minipage}}
  \caption{The Scaling algorithm.}
  \label{alg:scaling}
\end{figure}

\noindent
We define
\begin{align}
  R \eqdef \lceil \log_2 (2T/c) \rceil_+ \quad \textrm{and} \quad U_r
  \eqdef \frac{Y}{X} \frac{2^r}{\sqrt{T \ln(2d)}} , \quad\mbox{for
  }r=0,\ldots,R~,\label{eq:ur}
\end{align}
where $c>0$ is a known absolute constant and
\[
\lceil x \rceil_+ \eqdef \min\bigl\{k \in \N: k \geq x \bigr\} \quad
\textrm{for all~} x \in \R~.
\]
The Scaling algorithm of Figure~\ref{alg:scaling} works as follows.
We have access to a sub-algorithm $\mathcal{A}(U)$ which we run simultaneously for all $U=U_r$, $r=0,\ldots,R$. Each instance of the sub-algorithm $\mathcal{A}(U_r)$ performs online linear regression on the $\ell^1$-ball $B_1(U_r)$.
We employ an exponentially weighted forecaster to aggregate these $R+1$ sub-algorithms to perform online linear
regression simultaneously on the balls $B_1(U_0),\ldots,B_1(U_R)$. The following regret bound follows by exp-concavity of the square loss.

\vspace{0.2cm}
\begin{theorem}\label{thm:scaling}
  Suppose that $X,Y>0$ are known. Let $c,c' > 0$ be two absolute
  constants.  Suppose that for all $U>0$, we have access to a
  sub-algorithm $\mathcal{A}(U)$ with regret against $B_1(U)$ of
  at most
  \begin{align}\label{eq:833}
    & c U XY \sqrt{T \ln(2d)} + c'Y^2 \quad\mbox{for }T \geq T_0~,
  \end{align}
  uniformly over all sequences $(\bx_t)$ and $(y_t)$ bounded by $X$ and $Y$.
  Then, for a known $T \geq T_0$, the Scaling algorithm
  with $\eta = 1/(8Y^2)$ satisfies
  \begin{align}
    \sum_{t=1}^T (y_t - \hat y_t)^2 &\leq \inf_{\bu \in \mathbb R^d}
    \left\{ \sum_{t=1}^T (y_t - \bu \cdot \bx_t)^2 + 2c \Norm{\bu}_1 XY \sqrt{T \ln(2d)} \right\}\nonumber\\
    & \quad + 8 Y^2 \ln\bigl(\lceil \log_2 (2T/c) \rceil_+ +1\bigr) +
    (c+c') Y^2.\label{eq:claim1}
  \end{align}
  In particular, for every $U > 0$,
  \begin{align*}
    \sum_{t=1}^T (y_t - \hat y_t)^2 &\leq \inf_{\bu \in B_1(U)}
    \left\{ \sum_{t=1}^T (y_t - \bu \cdot \bx_t)^2 \right\} + 2c UXY \sqrt{T \ln (2d)}\\
    & \quad + 8 Y^2 \ln\bigl(\lceil \log_2 (2T/c) \rceil_+ +1\bigr) +
    (c+c') Y^2.
  \end{align*}
\end{theorem}

\vspace{0.2cm}
\begin{remark}
  By Remark~\ref{rem:chapL1-simpler-analysis} the LEG algorithm satisfies assumption~(\ref{eq:833})
  with $T_0 = \ln(2 d)$, $c \eqdef 9 c_1=72 \bigl(\sqrt{2}+1\bigr)$,
  and $c' \eqdef c_2=4\left(1+1/\sqrt{2}\right)^2$.
\end{remark}

\vspace{0.1cm}
\begin{proof}
  Since the Scaling algorithm is an exponentially weighted average
  forecaster (with clipping) applied to the $R+1$ experts
  $\mathcal{A}(U_r) = \bigl(\hat{y}^{(r)}_t\bigr)_{t \geq 1}$,
  $r=0,\ldots,R$, we have, by Lemma~\ref{lem:EWA-exp-concave} in
  \refapx{sec:chapL1-lemmas},
  \begin{align}
    &\sum_{t=1}^T (y_t - \hat y_t)^2 \leq \min_{r=0,\ldots,R} \sum_{t=1}^T \left(\hat{y}_t^{(r)} - \hat y_t\right)^2 + 8Y^2 \ln (R+1) \nonumber \\
    &\leq \min_{r=0,\ldots,R} \left\{ \inf_{\bu \in B_1(U_r)} \left\{
        \sum_{t=1}^T (y_t - \bu \cdot \bx_t)^2 \right\} + c U_r XY
      \sqrt{T \ln(2d)} \right\} + z~, \label{eq:xx1}
  \end{align}
  where the last inequality follows by assumption~(\ref{eq:833}), and
  where we set
  \[
  z \eqdef 8Y^2 \ln (R+1) + c' Y^2~.
  \]

\noindent
Let $\bu_T^* \in \arg\min_{\bu \in \mathbb R^d} \left\{ \sum_{t=1}^T (y_t - \bu \cdot \bx_t)^2 + 2c \Norm{\bu}_1 XY \sqrt{T \ln(2d)} \right\}$. Next, we proceed by considering three cases: $U_{0} < \Norm{\bu_T^*}_1 < U_{R}$, $\Norm{\bu_T^*}_1 \leq U_{0}$, and $\Norm{\bu_T^*}_1 \geq U_{R}$.\\

\noindent
  \textbf{Case 1}: $U_{0} < \Norm{\bu_T^*}_1 < U_{R}$. Let $r^* \eqdef
  \min \bigl\{r =0,\ldots,R:U_r \geq \Norm{\bu_T^*}_1\bigr\}$. Note
  that $r^* \geq 1$ since $\Norm{\bu_T^*}_1 > U_0$. By~(\ref{eq:xx1})
  we have
  \begin{align*}
    \sum_{t=1}^T (y_t - \hat y_t)^2 &\leq \!\inf_{\bu \in
      B_1(U_{r^*})} \!\left\{ \sum_{t=1}^T (y_t - \bu \cdot \bx_t)^2
      \!\right\} + c U_{r^*} XY
    \sqrt{T \ln(2d)} + z \nonumber\\
    &\leq \sum_{t=1}^T (y_t - \bu_T^* \cdot \bx_t)^2 + 2c
    \Norm{\bu_T^*}_1 XY \sqrt{T \ln(2d)} + z~,
  \end{align*}
  where the last inequality follows from $\bu_T^* \in B_1(U_{r^*})$ and from the fact that $U_{r^*} \leq 2\Norm{\bu_T^*}_1$ (since, by definition of $r^*$, $\Norm{\bu_T^*}_1 > U_{r^*-1}= U_{r^*}/2$). Finally, we obtain \eqref{eq:claim1} by definition of $\bu_T^*$ and $z \eqdef 8Y^2 \ln (R+1) + c' Y^2$. \\

\noindent
  \textbf{Case 2}: $\Norm{\bu_T^*}_1 \leq U_{0}$.  By \eqref{eq:xx1}
  we have
  \begin{align}
    \sum_{t=1}^T (y_t - \hat y_t)^2 &\leq \left\{ \sum_{t=1}^T (y_t -
      \bu_T^* \cdot \bx_t)^2 + c U_0 XY \sqrt{T \ln(2d)} \right\} +
    z~, \label{eq:b}
  \end{align}
    which yields \eqref{eq:claim1} by the equality $c U_{0} XY \sqrt{T \ln(2d)} = cY^2$ (by definition of $U_0$), by adding the nonnegative quantity $2 c \Norm{\bu^*_T}_1 XY \sqrt{T \ln(2d)}$, and by definition of $\bu^*_T$ and $z$. \\

\noindent
\textbf{Case 3}: $\Norm{\bu_T^*}_1 \geq U_{R}$.  By construction, we have $\hat y_t\in[-Y,Y]$, and by assumption, we have $y_t\in[-Y,Y]$, so that
  \begin{align*}
    \sum_{t=1}^T (y_t - \hat y_t)^2 &\leq 4 Y^2 T \leq \sum_{t=1}^T
    (y_t - \bu_T^*
    \cdot \bx_t)^2 + 2c U_{R} XY \sqrt{T \ln(2d)}\\
    & \leq \sum_{t=1}^T (y_t - \bu_T^* \cdot \bx_t)^2 + 2c
    \Norm{\bu_T^*}_1 XY \sqrt{T \ln(2d)}~,
  \end{align*}
  where the second inequality follows by $2c U_{R} XY \sqrt{T \ln(2d)} = 2cY^2 2^R \geq 4 Y^2 T$ (since $2^R \geq 2T/c$ by definition of $R$), and the last inequality uses the assumption $\Norm{\bu_T^*}_1 \geq U_{R}$. We finally get \eqref{eq:claim1} by definition of $\bu_T^*$.

  This concludes the proof of the first claim~(\ref{eq:claim1}). The
  second claim follows by bounding $\Norm{\bu}_1 \leq U$.
\end{proof}

\section{Extension to a fully adaptive algorithm}
\label{sec:dis}

The Scaling algorithm of Section~\ref{sec:algo2} uses prior knowledge of $Y$, $Y/X$, and $T$. 
In order to obtain a fully automatic algorithm, 
we need to adapt efficiently to these quantities.
Adaptation to $Y$ is possible via a technique already used for the LEG algorithm, \ie, by updating the clipping range $B_t$ based on the past observations $|y_s|$, $s \leq t-1$.

In parallel to adapting to $Y$, adaptation to $Y/X$ can be carried out as follows. 
We replace the exponential sequence $\{U_0, \ldots, U_R\}$ by another exponential sequence $\{U'_0, \ldots, U'_{R'}\}$:
\begin{equation}
\label{eqn:newgrid}
U'_r \eqdef \frac{1}{T^k} \frac{2^r}{\sqrt{T \ln(2d)}}~, \quad r=0,\ldots,R'~,
\end{equation}
where $R' \eqdef R+ \bigl\lceil \log_2 T^{2k} \bigr\rceil = \lceil \log_2 (2T/c) \rceil_+ + \bigl\lceil \log_2 T^{2k} \bigr\rceil$, and where $k>1$ is a fixed constant.
On the one hand, for $T \geq T_0 \eqdef \max\bigl\{(X/Y)^{1/k},(Y/X)^{1/k}\bigr\}$, we have (cf. \eqref{eq:ur} and \eqref{eqn:newgrid}),
\begin{equation*}
\label{eqn:grid-extension}
[U_0,U_R] \subset [U'_0, U'_ {R'}]~.
\end{equation*}
Therefore, the analysis of Theorem~\ref{thm:scaling} 
applied to the grid $\{U'_0, \ldots, U_{R'}\}$ yields\footnote{The proof remains the same by replacing $8Y^2\ln(R+1)$ with $8Y^2\ln(R'+1)$.} a regret bound of the order of $UXY\sqrt{T \ln d} + Y^2\ln(R'+1)$.
On the other hand, clipping the predictions to $[-Y,Y]$ ensures the crude regret bound $4 Y^2 T_0$ for small $T < T_0$.
Hence, the overall regret for all $T \geq 1$ is of the order of
\[
UXY\sqrt{T \ln d} + Y^2 \ln (k\ln T) + Y^2 \max\bigl\{(X/Y)^{1/k},(Y/X)^{1/k}\bigr\}~.
\]

Adaptation to an unknown time horizon $T$ can be carried out via a standard doubling trick on $T$. However, to avoid restarting the algorithm repeatedly,
we can use a time-varying exponential sequence $\{U'_{-R'(t)}(t), \ldots, U'_{R'(t)}(t)\}$ 
where $R'(t)$ grows at the rate of $k \ln(t)$. This gives\footnote{Each time the exponential sequence $(U'_r)$ expands, the weights assigned to the existing points $U'_r$ are appropriately reassigned to the whole new sequence.} us an algorithm that is fully automatic in the parameters $U$, $X$, $Y$ and $T$. 
In this case, we can show that the regret is of the order of
\[
UXY\sqrt{T \ln d} + Y^2 k \ln(T) + Y^2 \max\left\{\bigl(\sqrt{T}X/Y\bigr)^{1/k},\bigl(Y/(\sqrt{T}X)\bigr)^{1/k}\right\},
\]
where the last two terms are negligible when $T \to +\infty$ (since $k > 1$).

%%%%%%%%%%%%%%%%%%%%%%%%%%%%%%%%%%%%%%%%%%%%%%%%%%
\section*{Acknowledgments}
%%%%%%%%%%%%%%%%%%%%%%%%%%%%%%%%%%%%%%%%%%%%%%%%%%

The authors would like to thank Gilles Stoltz for his valuable comments and suggestions, as well as two anonymous reviewers for their insightful feedback. This work was supported in part by French National Research Agency (ANR, project EXPLO-RA, ANR-08-COSI-004) and the PASCAL2 Network of Excellence under EC grant {no.} 216886. J.\ Y.\ Yu was partly supported by a fellowship from Le Fonds qu\'eb\'ecois de la recherche sur la nature et les technologies.

An extended abstract of the present paper appeared in the \textit{Proceedings of the 22nd International Conference on Algorithmic Learning Theory (ALT'11)}.

%%%%%%%%%%%%%%%%%%%%%%%%%%%%%%%%%%%%%%%%%%%%%%%%%%
%% APPENDIX
%%%%%%%%%%%%%%%%%%%%%%%%%%%%%%%%%%%%%%%%%%%%%%%%%%

\appendix
\renewcommand*{\thesection}{\appendixname\Alph{section}}

%%%%%%%%%%%%%%%%%%%%%%%%%%%%%%%%%%%%%%%%%%%%%%%%%%%%%%%%%%%%%%%%%%%%%
\section{Proofs}
%%%%%%%%%%%%%%%%%%%%%%%%%%%%%%%%%%%%%%%%%%%%%%%%%%%%%%%%%%%%%%%%%%%%%

\subsection{Proof of Theorem~\ref{thm:lowerbound}}
\label{sec:proof-thm2}

To prove Theorem~\ref{thm:lowerbound}, we perform a reduction to the stochastic batch setting (via the standard online to batch trick), and employ a version of the lower bound proved in \cite{Tsy-03-OptimalRates} for convex aggregation. \\

\noindent
We first need the following notations. Let $T \in \N^*$. Let $(S,\mu)$ be a probability space for which we can find an orthonormal family\footnote{An example is given by $S=[-\pi,\pi]$, $\mu(\mbox{d} x)=\mbox{d} x / (2 \pi)$, and $\varphi_j(x) = \sqrt{2} \sin(j x)$ for all $1 \leq j \leq d$ and $x \in [-\pi, \pi]$. We will use this particular case later.} $(\varphi_j)_{1 \leq j \leq d}$ with $d$ elements in the space of square-integrable functions on $S$, which we denote by $\Law^2(S,\mu)$ thereafter. For all $\bu \in \R^d$ and $\gamma,\sigma > 0$, denote by $\Prob^{\gamma,\sigma}_{\bu}$ the joint law of the {i.i.d.} sequence $(X_t,Y_t)_{1 \leq t \leq T}$ such that
\begin{equation}
\label{eqn:chapL1-lowerbound-model}
Y_t = \gamma \varphi_{\bu}(X_t) + \sigma \epsilon_t \in \R~,
\end{equation}
where $\varphi_{\bu} \eqdef \sum_{j=1}^d u_j \varphi_j$, where the $X_t$ are {i.i.d} points in $S$ drawn from $\mu$, and where the $\epsilon_t$ are {i.i.d} standard
Gaussian random variables such that $(X_t)_{1 \leq t \leq T}$ and $(\epsilon_t)_{1 \leq t \leq T}$ are independent. \\
\ \\
The next lemma is a direct adaptation of \cite[Theorem~2]{Tsy-03-OptimalRates}, which we state with our notations in a slightly more precise form (we make clear how the lower bound depends on the noise level $\sigma$ and the signal level~$\gamma$).

\begin{lemma}[An extension of Theorem~2 of \cite{Tsy-03-OptimalRates}]
  \label{lem:Tsy-convex-lower}
	\ \\
  Let $d,T \in \N^*$ and $\gamma, \sigma > 0$. Let $(S,\mu)$ be a probability space for which we can find an orthonormal family $(\varphi_j)_{1 \leq j \leq d}$ in $\Law^2(S,\mu)$, and consider the Gaussian linear model \eqref{eqn:chapL1-lowerbound-model}. Then there exist absolute constants $c_4,c_5,c_6,c_7>0$ such that
    \begin{align*}
    & \inf_{\hat{f}_T} \sup_{\substack{\bu \in \R_+^d \\ \sum_j u_j \leq 1}}\!\Biggl\{ \E_{\Prob^{\gamma,\sigma}_{\bu}}\!\Norm{\hat{f}_T- \gamma \varphi_{\bu}}_{\mu}^2 \!\Biggr\} \\
    & \quad \geq \left\{\begin{array}{ll}
        c_4 \frac{d \sigma^2}{T} & \textrm{if} \quad \frac{d}{\sqrt{T}} \leq c_5 \frac{\gamma}{\sigma}~, \\
        c_6 \gamma \sigma \sqrt{\frac{1}{T} \ln \left(1+\frac{d\sigma}{\sqrt{T}\gamma}\right)} & \textrm{if} \quad c_5 \frac{\gamma}{\sigma} < \frac{d}{\sqrt{T}} \leq c_7 \frac{\gamma d}{\sigma \sqrt{\ln(1+d)}}~,
      \end{array} \right.
  \end{align*}
  where the infimum is taken over all estimators\footnote{As usual, an estimator is a measurable function of the sample $(X_t,Y_t)_{1 \leq t \leq T}$, but the dependency on the sample is omitted.} $\hat{f}_T:S \to \R$, where the supremum is taken over all nonnegative vectors with total mass at most $1$, and where $\Norm{f}_{\mu}^2 \eqdef  \int_S f(x)^2 \mu(\dd x)$ for all measurable functions $f:S \to \R$.
\end{lemma}

\vspace{0.3cm}
Note that the lower bound we stated in Theorem~\ref{thm:lowerbound} is very similar to $T$ times the above lower bound with $\gamma \sim X$ and $\sigma \sim Y$ (recall that $\kappa \eqdef \sqrt{T}UX/(2dY)$). The main difference is that the latter holds for unbounded observations, while we need bounded observations $y_t$, $1 \leq t \leq T$. A simple concentration argument will show that these observations lie in $[-Y,Y]$ with high probability, which will yield the desired lower bound. The proof of Theorem~\ref{thm:lowerbound} thus consists of the following steps:
\begin{itemize}
\item step 1: reduction to the stochastic batch setting;
\item step 2: application of Lemma~\ref{lem:Tsy-convex-lower};
\item step 3: concentration argument.
\end{itemize}

\vspace{0.2cm}
\begin{proofref}{Theorem~\ref{thm:lowerbound}}
  We first assume that $\sqrt{\ln(1+2d)}/\bigl(2 d \sqrt{\ln 2}\bigr) \leq \kappa \leq 1$. The case when $\kappa > 1$ will easily follow from the monotonicity of the minimax regret in $\kappa$ (see the end of the proof). We set
  \begin{equation}
  \label{eqn:def-T-U-X}
  T \eqdef 1+ \bigl\lceil (4d\kappa)^2 \bigr\rceil~, \quad U \eqdef 1~, \quad \textrm{and} \quad X \eqdef \frac{2 d \kappa Y}{\sqrt{T}}~,
	\end{equation}
	so that $T \geq 2$, $\sqrt{T} U X/(2 d Y) = \kappa$, and $X \leq Y/2$ (since $\sqrt{T} \geq 4d\kappa$). \\
	
	\noindent
	{\bf Step 1}: reduction to the stochastic batch setting. \\
	First note that by clipping to $[-Y,Y]$, we have
  \begin{align}
    & \inf_{(\tilde{f}_t)_t} \sup_{\substack{\Norm{\bx_t}_{\infty} \!\leq X \\ \Abs{y_t}\leq Y}} \Biggl\{ \sum_{t=1}^T \bigl(y_t - \tilde{f}_t(\bx_t)\bigr)^2 - \inf_{\Norm{\bu}_1 \leq U} \sum_{t=1}^T (y_t - \bu \cdot \bx_t)^2 \Biggr\} \nonumber \\
    & = \inf_{\substack{(\tilde{f}_t)_t\\|\tilde{f}_t|\leq Y}}
    \sup_{\substack{\Norm{\bx_t}_{\infty} \!\leq X \\ \Abs{y_t}\leq Y}} \Biggl\{ \sum_{t=1}^T \bigl(y_t - \tilde{f}_t(\bx_t)\bigr)^2 -
    \inf_{\Norm{\bu}_1 \leq U} \sum_{t=1}^T (y_t - \bu \cdot \bx_t)^2
    \Biggr\}, \label{eqn:lowerbound-clipping}
  \end{align}
	where the first infimum is taken over all online forecasters\footnote{\label{foot:forecaster-def}Recall that an online forecaster is a sequence of functions $(\tilde{f}_t)_{t \geq 1}$, where $\tilde{f}_t:\R^d \times (\R^d \times \R)^{t-1} \to \R$ maps at time $t$ the new input $\bx_t$ and the past data $(\bx_s,y_s)_{1 \leq s \leq t-1}$ to a prediction $\tilde{f}_t\bigl(\bx_t;(\bx_s,y_s)_{1 \leq s \leq t-1}\bigr)$. However, unless mentioned otherwise, we omit the dependency in $(\bx_s,y_s)_{1 \leq s \leq t-1}$, and only write $\tilde{f}_t(\bx_t)$.} $(\tilde{f}_t)_t$, where the second infimum is restricted to online forecasters $(\tilde{f}_t)_t$ which output predictions in $[-Y,Y]$, and where both suprema are taken over all individual sequences $(\bx_t,y_t)_{1 \leq t \leq T} \in (\R^d \times \R)^T$ such that $|y_1|, \ldots, |y_T| \leq Y$ and $\Norm{\bx_1}_{\infty}, \ldots, \Norm{\bx_T}_{\infty} \leq X$.  \\

	Next we use the standard online to batch conversion to bound from below the right-hand side of \eqref{eqn:lowerbound-clipping} by $T$ times the lower bound of Lemma~\ref{lem:Tsy-convex-lower}, which we apply to the particular case where $S=[-\pi,\pi]$, where $\mu(\mbox{d} x)=\mbox{d} x / (2 \pi)$, and where $\varphi_j(x) = \sqrt{2} \sin(j x)$ for all $1 \leq j \leq d$ and $x \in [-\pi, \pi]$. Let
	\begin{equation}
  \label{eqn:def-gamma-sigma}
  \gamma \eqdef c_8 X \quad \textrm{and} \quad \sigma \eqdef \frac{c_9 Y}{\sqrt{\ln T}}~,
  \end{equation}
  for some absolute constants $c_8,c_9>0$ to be chosen by the analysis. \\

	\noindent
	Let $(\tilde{f}_t)_{t \geq 1}$ be any online forecaster whose predictions lie in $[-Y,Y]$, and consider the estimator $\hat{f}_T$ defined for each sample $(X_t,Y_t)_{1 \leq t \leq T}$ and each new input $X'$ by
	\begin{equation}
	\label{eqn:lowerbound-standardTrick}
	\hat{f}_T\Bigl(X';(X_t,Y_t)_{1 \leq t \leq T}\Bigr) \eqdef \frac{1}{T} \sum_{t=1}^T \tilde{f}_t\Bigl(\gamma \bphi(X');(\gamma \bphi(X_s),Y_s)_{1 \leq s \leq t-1}\Bigr)~,
	\end{equation}
	where $\bphi \eqdef (\varphi_1,\ldots,\varphi_d)$, and where we explicitely wrote all the dependencies$^{\ref{foot:forecaster-def}}$ of the $\tilde{f}_t$, $t=1,\ldots,T$.

	\noindent
	Take $\bu^* \in \R_+^d$ achieving the supremum\footnote{If the supremum in Lemma~\ref{lem:Tsy-convex-lower} is not achieved, then we can instead take an $\epsilon$-almost-maximizer for any $\epsilon >0$. Letting $\epsilon \rightarrow 0$ in the end will conclude the proof.} in Lemma~\ref{lem:Tsy-convex-lower} for the estimator $\hat{f}_T$.	Note that ${\norm[\bu^*]_1 \leq 1}$. Besides, consider the {i.i.d.}\ random sequence $(\bx_t,y_t)_{1 \leq t \leq T}$ in $\R^d \times \R$ defined for all $t=1,\ldots,T$ by
\begin{equation}
\label{eqn:chapL1-lowerbound-model-reduction}
\bx_t \eqdef \bigl(\gamma \varphi_1(X_t), \ldots, \gamma \varphi_d(X_t)\bigr) \quad \textrm{and} \quad y_t \eqdef \gamma \varphi_{\bu^*}(X_t) + \sigma \epsilon_t~,
\end{equation}
where $\varphi_{\bu^*} \eqdef \sum_{j=1}^d u^*_j \varphi_j$ (so that $y_t = \bu^* \cdot \bx_t + \sigma \epsilon_t$ for all $t$), where the $X_t$ are {i.i.d} points in $[-\pi,\pi]$ drawn from the uniform distribution $\mu(\mbox{d} x)=\mbox{d} x / (2 \pi)$, and where the $\epsilon_t$ are {i.i.d} standard Gaussian random variables such that $(X_t)_t$ and $(\epsilon_t)_t$ are independent. All the expectations below are thus taken with respect to the probability distribution $\Prob^{\gamma,\sigma}_{\bu^*}$. \\

By standard manipulations (\eg, using the tower rule and Jensen's inequality), we get the following lower bound. A detailed proof can be found after the proof of the present theorem (page~\pageref{proof:chapL1-lower-Batchconversion}).

\begin{lemma}[Reduction to the batch setting]
\label{lem:chapL1-lower-Batchconversion}
\ \\
With $(\tilde{f}_t)_{1 \leq t \leq T}$, $\hat{f}_T$, and $\bu^*$ defined above, we have
\[
\E\!\left[\sum_{t=1}^T \bigl(y_t-\tilde{f}_t(\bx_t)\bigr)^2 - \inf_{\Norm{\bu}_1 \leq 1} \sum_{t=1}^T \bigl(y_t-\bu \cdot \bx_t\bigr)^2 \right] \geq T \, \E\Norm{\hat{f}_T - \gamma \varphi_{\bu^*}}^2_{\mu}~.
\]
\end{lemma}

\vspace{0.5cm}
\noindent
	{\bf Step 2}: application of Lemma~\ref{lem:Tsy-convex-lower}. \\
	Next we use Lemma~\ref{lem:Tsy-convex-lower} to prove that, for some absolute constants $c_9,c_{11}>0$,
  \begin{align}
    & T \, \E \Norm{\hat{f}_T- \gamma \varphi_{\bu^*}}_{\mu}^2 \geq \frac{c_{11} c_9^2}{\ln\bigl(2+16d^2\bigr)} d Y^2 \kappa \sqrt{\ln(1+1/\kappa)}~. \label{eqn:tsy-lowerbound2}
  \end{align}	
	 By Lemma~\ref{lem:Tsy-convex-lower} and by definition of $\bu^*$, we have
  \begin{align}
    \E \Norm{\hat{f}_T- \gamma \varphi_{\bu^*}}_{\mu}^2 & \geq \left\{\begin{array}{ll}
        c_4 \frac{d \sigma^2}{T} & \textrm{if} \quad \frac{d}{\sqrt{T}} \leq c_5 \frac{\gamma}{\sigma}~, \\
        c_6 \gamma \sigma \sqrt{\frac{1}{T} \ln \left(1+\frac{d\sigma}{\sqrt{T}\gamma}\right)} & \textrm{if} \quad c_5 \frac{\gamma}{\sigma} < \frac{d}{\sqrt{T}} \leq \frac{c_7 \gamma d}{\sigma \sqrt{\ln(1+d)}}~.
      \end{array} \right. \nonumber \\
    & \geq \left\{\begin{array}{ll}
        \frac{c_4 c_9^2}{T(\ln T)} d Y^2 & \textrm{if} \quad \frac{d}{\sqrt{T}} \leq c_5 \frac{\gamma}{\sigma}~, \\
        \frac{c_6 c_8 c_9}{\sqrt{\ln T}} U X Y \sqrt{\frac{1}{T} \ln \left(1+\frac{c_9 dY}{c_8 \sqrt{T(\ln T)}UX}\right)} & \textrm{if} \quad c_5 \frac{\gamma}{\sigma} < \frac{d}{\sqrt{T}} \leq \frac{c_7 \gamma d}{\sigma \sqrt{\ln(1+d)}}~,
      \end{array} \right. \label{eqn:tsy-lowerbound1}
  \end{align}
where the last inequality follows from (\ref{eqn:def-gamma-sigma}) and from $U=1$.\\

\noindent
The above lower bound is only meaningful if the following condition holds true:
\begin{equation}
\label{eqn:condition-true}
\frac{d}{\sqrt{T}} \leq \frac{c_7 \gamma d}{\sigma \sqrt{\ln(1+d)}}~.
\end{equation}
But, by definition of $T \eqdef 1+ \bigl\lceil (4d\kappa)^2 \bigr\rceil$ and by the assumption $\sqrt{\ln(1+2d)}/\bigl(2 d \sqrt{\ln 2}\bigr) \leq \kappa$, elementary manipulations show that \eqref{eqn:condition-true} actually holds true whenever\footnote{By definition of $\gamma$ and $\sigma$, \eqref{eqn:condition-true} is equivalent to $T \ln T \geq c_9^2/(c_7^2 c_8^2) (Y/X)^2 \ln(1+d)$. But by definition of $X$ and by the assumption $\kappa \geq \sqrt{\ln(1+2d)}/(2d \sqrt{\ln 2})$, we have $Y/X \leq 1/c_{10}$. Therefore, \eqref{eqn:condition-true} is implied by $T \ln T \geq c_9^2/(c_7^2 c_8^2 c_{10}^2) \ln(1+d)$, which in turn is implied by the condition $c_9 \leq c_7 c_8 c_{10}$ (by definition of $T$).} $c_9 \leq c_7 c_8 c_{10}$, where $c_{10} \eqdef \frac{1}{2} \inf_{x \geq 2\sqrt{\frac{\ln 3}{\ln 2}}} \left\{\frac{x}{\sqrt{1+\lceil x^2\rceil}}\right\}$ (note that $c_{10} > 0$).\\

\noindent
Therefore, if $c_9 \leq c_7 c_8 c_{10}$, then \eqref{eqn:tsy-lowerbound1} entails that
\begin{align}
\label{eqn:chapL1-lowerbound-1}
\E \Norm{\hat{f}_T- \gamma \varphi_{\bu^*}}_{\mu}^2 \geq \min \left\{\frac{c_4 c_9^2}{T(\ln T)} d Y^2, \, \frac{c_6 c_8 c_9}{\sqrt{\ln T}} U X Y \sqrt{\frac{1}{T} \ln \left(1+\frac{c_9 dY}{c_8 \sqrt{T(\ln T)}UX}\right)} \, \right\}~.
\end{align}

\noindent
Moreover, note that if $c_9 \leq c_8 2 \sqrt{\ln 2}$, then $c_8 \geq c_9 / (2\sqrt{\ln 2}) \geq c_9 / (2\sqrt{\ln T})$.  In this case, since $x \mapsto x\sqrt{\ln(1+A/x)}$ is nondecreasing on $\R_+^*$ for all $A>0$, we can replace $c_8$ with $c_9 / (2\sqrt{\ln T})$ in the next expression and get
\begin{align*}
& \frac{c_6 c_8 c_9}{\sqrt{\ln T}} U X Y \sqrt{\frac{1}{T} \ln \left(1+\frac{c_9 dY}{c_8 \sqrt{T(\ln T)}UX}\right)} \\
& \quad \geq \frac{c_6 c_9^2}{2 \ln T} U X Y \sqrt{\frac{1}{T} \ln \left(1+\frac{2 dY}{\sqrt{T}UX}\right)} = \frac{c_6 c_9^2}{T(\ln T)} d Y^2 \kappa \sqrt{\ln(1+1/\kappa)}~,
\end{align*}
where we used the definition of $\kappa \eqdef \sqrt{T}UX/(2dY)$. \\
In the sequel we will choose the absolute constants $c_8$ and $c_9$ such that
\begin{equation}
\label{eqn:choice-constants1}
c_9 \leq c_7 c_8 c_{10} \quad \textrm{and} \quad c_9 \leq c_8 2 \sqrt{\ln 2}~.
\end{equation}
Therefore, by the above remarks, by the fact that $\ln T \eqdef \ln\bigl(1+ \lceil (4d\kappa)^2 \rceil\bigr) \leq \ln\bigl(2+16d^2\bigr)$ (since $\kappa \leq 1$ by assumption), and multiplying both sides of~\eqref{eqn:chapL1-lowerbound-1} by $T$, we get
  \begin{align*}
    T \, \E \Norm{\hat{f}_T- \gamma \varphi_{\bu^*}}_{\mu}^2 & \geq \min\left\{\frac{c_4 c_9^2}{\ln\bigl(2+16d^2\bigr)} d Y^2, \, \frac{c_6 c_9^2}{\ln\bigl(2+16d^2\bigr)} d Y^2 \kappa \sqrt{\ln(1+1/\kappa)} \right\} \\
    & \geq \frac{c_{11} c_9^2}{\ln\bigl(2+16d^2\bigr)} d Y^2 \kappa \sqrt{\ln(1+1/\kappa)}~,
  \end{align*}
  where we set $c_{11} \eqdef \min\bigl\{c_4/\sqrt{\ln 2},c_6\bigr\}$, and where we used the fact that $x \mapsto x \sqrt{\ln(1+1/x)}$ is nondecreasing on $\R_+^*$, so that its value at $x=\kappa \leq 1$ is smaller than $\sqrt{\ln 2}$. This concludes the proof of \eqref{eqn:tsy-lowerbound2}. \\

\noindent
	Combining Lemma~\ref{lem:chapL1-lower-Batchconversion} and \eqref{eqn:tsy-lowerbound2}, we get
\begin{align}
\E\!\left[\sum_{t=1}^T \bigl(y_t-\tilde{f}_t(\bx_t)\bigr)^2 - \inf_{\Norm{\bu}_1 \leq 1} \sum_{t=1}^T \bigl(y_t-\bu \cdot \bx_t\bigr)^2 \right] \geq \frac{c_{11} c_9^2}{\ln\bigl(2+16d^2\bigr)} d Y^2 \kappa \sqrt{\ln(1+1/\kappa)}~. \label{eqn:lowerbound-standardTrick3}
\end{align}

\vspace{0.2cm}
\noindent
{\bf Step 3}: concentration argument. \\
At this stage it would be tempting to conclude by using \eqref{eqn:lowerbound-standardTrick3} that since the expectation is lower bounded,  then there is at least one individual sequence with the same lower bound. However, we have no boundedness guarantee about such individual sequence since the random observations $y_t$ lie outside of $[-Y,Y]$ with positive probability. Next we prove that the probability of the event
\[
\cA \eqdef \bigcap_{t=1}^T \bigl\{|y_t| \leq Y \bigr\}
\]
is actually close to $1$, and that
\begin{equation}
\label{eqn:chapL1-lowerbound-2}
\E\!\left[\indicator{\cA} \left(\sum_{t=1}^T \bigl(y_t-\tilde{f}_t(\bx_t)\bigr)^2 - \inf_{\Norm{\bu}_1 \leq 1} \sum_{t=1}^T \bigl(y_t-\bu \cdot \bx_t\bigr)^2 \right)\right] \geq \frac{1}{2} \frac{c_{11} c_9^2}{\ln\bigl(2+16d^2\bigr)} d Y^2 \kappa \sqrt{\ln(1+1/\kappa)}~.
%\E\Biggl[\indicator{\cA} \Biggl(\, \underbrace{\sum_{t=1}^T \bigl(y_t-\tilde{f}_t(\bx_t)\bigr)^2}_{\eqdef \, \hat{L}_T}-\inf_{\Norm{\bu}_1 \leq 1} \underbrace{\sum_{t=1}^T \bigl(y_t-\bu \cdot \bx_t\bigr)^2}_{ \eqdef \, L_T(\bu)} \, \Biggr)\Biggr] \geq \frac{1}{2} \frac{c_{11} c_9^2}{\ln\bigl(2+16d^2\bigr)} d Y^2 \kappa \sqrt{\ln(1+1/\kappa)}~.
\end{equation}
(Note a missing factor of $2$ between \eqref{eqn:lowerbound-standardTrick3} and \eqref{eqn:chapL1-lowerbound-2}.) The last lower bound will then enable us to conclude the proof of this theorem. \\

\noindent
Set $\hat{L}_T \eqdef \sum_{t=1}^T \bigl(y_t-\tilde{f}_t(\bx_t)\bigr)^2$ and $L_T(\bu) \eqdef \sum_{t=1}^T \bigl(y_t-\bu \cdot \bx_t\bigr)^2$ for all $\bu \in \R^d$. Denote by $\cA^c$ the complement of $\cA$, and by $\indicator{\cA}$ and $\indicator{\cA^c}$ the corresponding indicator functions. By the equality $\indicator{\cA}=1-\indicator{\cA^c}$, we have
\begin{align}
\E\!\left[\indicator{\cA}\left(\hat{L}_T-\inf_{\Norm{\bu}_1 \leq 1} L_T(\bu)\right)\right] & = \E\!\left[\hat{L}_T-\inf_{\Norm{\bu}_1 \leq 1} L_T(\bu)\right] - \E\!\left[\indicator{\cA^c}\left(\hat{L}_T-\inf_{\Norm{\bu}_1 \leq 1} L_T(\bu)\right)\right]  \nonumber \\
& \geq \frac{c_{11} c_9^2}{\ln\bigl(2+16d^2\bigr)} d Y^2 \kappa \sqrt{\ln(1+1/\kappa)} - \E\!\left[\indicator{\cA^c}\hat{L}_T\right]~, \label{eqn:lowerbound-standardTrick4}
\end{align}
where the last inequality follows by \eqref{eqn:lowerbound-standardTrick3} and by the fact that $L_T(\bu) \geq 0$ for all $\bu \in \R^d$. The rest of the proof is dedicated to upper bounding the above quantity $\E\bigl[\indicator{\cA^c}\hat{L}_T\bigr]$ by half the term on its left. This way, we will have proved \eqref{eqn:chapL1-lowerbound-2}. \\

\noindent
First note that
\begin{align}
\E\!\left[\indicator{\cA^c}\hat{L}_T\right] & \eqdef \E\!\left[\indicator{\cA^c}\sum_{t=1}^T  \bigl(y_t-\tilde{f}_t(\bx_t)\bigr)^2\right] \nonumber \\
& \leq \E\!\left[\indicator{\cA^c}\sum_{t=1}^T \left(4 Y^2 \indicatorB{|y_t| \leq Y} + \bigl(y_t-\tilde{f}_t(\bx_t)\bigr)^2 \indicatorB{|y_t|>Y}\right)\right] \label{eqn:concentration-1} \\
& \leq 4TY^2 \Prob\bigl(\cA^c\bigr) + \sum_{t=1}^T \E\!\left[\bigl(y_t-\tilde{f}_t(\bx_t)\bigr)^2 \, \indicatorB{|\varepsilon_t|>\frac{Y}{2\sigma}}\right]~, \label{eqn:concentration-2}
\end{align}
where \eqref{eqn:concentration-1} follows from the fact that the online forecaster $(\tilde{f}_t)_t$ outputs its predictions in $[-Y,Y]$. As for~\eqref{eqn:concentration-2}, note by definition of $y_t$ that $|y_t| \leq \Norm{\bu^*}_1 \gamma \Norm{\bphi(X_t)}_{\infty}  + \sigma |\varepsilon_t| \leq \gamma \sqrt{2} + \sigma |\varepsilon_t|$ since $\Norm{\bu^*}_1 \leq 1$ and $|\varphi_j(x)|  \eqdef |\sqrt{2} \sin(j x)| \leq \sqrt{2}$ for all $j=1,\ldots,d$ and $x \in \R$. Therefore, by definition of $\gamma \eqdef c_8 X$, and since $X \leq Y/2$ (by definition of $X$), we get $|y_t| \leq c_8 \sqrt{2} \, Y/2 + \sigma |\varepsilon_t| \leq Y/2 + \sigma |\varepsilon_t|$ provided that

\begin{equation}
\label{eqn:choice-constants2}
c_8 \leq \frac{1}{\sqrt{2}}~,
\end{equation}
which we assume thereafter. The above remarks show that $\{|y_t| > Y\} \subset \{|\varepsilon_t| > Y/(2\sigma)\}$, which entails \eqref{eqn:concentration-2}. By the same comments and since $|\tilde{f}_t| \leq Y$, we have, for all $t=1,\ldots,T$,
\begin{align}
\E\!\left[\bigl(y_t-\tilde{f}_t(\bx_t)\bigr)^2 \indicatorB{|\varepsilon_t|>\frac{Y}{2\sigma}}\right] & \leq \E\!\left[\bigl(Y/2 + \sigma|\varepsilon_t| + Y\bigr)^2 \indicatorB{|\varepsilon_t|>\frac{Y}{2\sigma}}\right] \nonumber \\
& \leq 2 \left(\frac{3Y}{2}\right)^2 \Prob\!\left(|\varepsilon_t| > \frac{Y}{2\sigma}\right) + 2 \sigma^2 \E\!\left[\varepsilon_t^2 \indicatorB{|\varepsilon_t|>\frac{Y}{2\sigma}}\right] \label{eqn:concentration-3} \\
& \leq \frac{9 Y^2}{2} \Prob\!\left(|\varepsilon_t| > \frac{Y}{2\sigma}\right) + 2 \sigma^2 \sqrt{3} \, \Prob^{1/2}\!\left(|\varepsilon_t| > \frac{Y}{2\sigma}\right) \label{eqn:concentration-4} \\
& \leq 9 Y^2 T^{-1/(8 c_9^2)} + 2 \frac{c_9^2 Y^2}{\ln 2} \sqrt{6} \, T^{-1/(16 c_9^2)}~, \label{eqn:concentration-5}
\end{align}
where we used the following arguments. Inequality \eqref{eqn:concentration-3} follows by the elementary inequality $(a+b)^2 \leq 2(a^2+b^2)$ for all $a,b \in \R$. To get \eqref{eqn:concentration-4} we used the Cauchy-Schwarz inequality and the fact that $\E\bigl[\varepsilon_t^4\bigr] = 3$ (since $\varepsilon_t$ is a standard Gaussian random variable). Finally, \eqref{eqn:concentration-5} follows by definition of $\sigma \eqdef c_9 Y / \sqrt{\ln T} \leq c_9 Y / \sqrt{\ln 2}$ and from the fact that, since $\varepsilon_t$ is a standard Gaussian random variable\footnote{We use a standard deviation inequality for subgaussian random variables; see, \eg, \cite[Equation~(2.5)]{Massart03StFlour} with $\sigma^2=1$.},
\begin{equation*}
\Prob\!\left(|\varepsilon_t| > \frac{Y}{2\sigma}\right) \leq 2 e^{-\frac{1}{2} \left(\frac{Y}{2\sigma}\right)^2} = 2 e^{-\frac{1}{2} \left(\frac{\sqrt{\ln T}}{2 c_9}\right)^2} = 2 T^{-1/(8 c_9^2)}~.
\end{equation*}

	\noindent
Using the fact that $\Prob\bigl(\cA^c\bigr) \leq \sum_{t=1}^T \Prob\bigl(|y_t|>Y\bigr) \leq \sum_{t=1}^T \Prob\bigl(|\varepsilon_t|>Y/(2\sigma)\bigr) \leq 2 T^{1-1/(8 c_9^2)}$ by the inequality above and substituting \eqref{eqn:concentration-5} in \eqref{eqn:concentration-2}, we get
\begin{align}
\E\!\left[\indicator{\cA^c}\hat{L}_T\right] & \leq 8 Y^2 T^{2-1/(8 c_9^2)} + 9 Y^2 T^{1-1/(8 c_9^2)} + \frac{2 c_9^2 \sqrt{6}}{\ln 2}  Y^2 T^{1-1/(16 c_9^2)} \nonumber \\
& \leq 8 Y^2 2^{2-1/(8 c_9^2)} + 9 Y^2 2^{1-1/(8 c_9^2)} + \frac{2 c_9^2 \sqrt{6}}{\ln 2} Y^2 2^{1-1/(16 c_9^2)}~, \label{eqn:concentration-7}
\end{align}
where the last inequality follows from the fact that $T^{\alpha} \leq 2^{\alpha}$ for all $\alpha < 0$ (since $T \geq 2$) and from a choice of $c_9$ such that $c_9 < 1/4$ (which we assume thereafter).\\

\noindent
In order to further upper bound $\E\bigl[\indicator{\cA^c}\hat{L}_T\bigr]$, we use the following technical lemma, which is proved after the proof of the present theorem (see page~\pageref{proof:chapL1-lower-technical}). It relies on the following elementary argument: since $d \, \kappa$ is large enough and since the left-hand side of the next inequality (Lemma~\ref{lem:chapL1-lower-technical}) decreases exponentially fast as $c_9 \to 0$, then this inequality holds true for all $c_9>0$ small enough.

\vspace{0.2cm}
\begin{lemma}
\label{lem:chapL1-lower-technical}
There exists an absolute constant $c_{13}>0$ such that, for all $c_9 \in (0,c_{13})$,
\[
8 Y^2 2^{2-1/(8 c_9^2)} + 9 Y^2 2^{1-1/(8 c_9^2)} + \frac{2 c_9^2 \sqrt{6}}{\ln 2} Y^2 2^{1-1/(16 c_9^2)} \leq \frac{1}{2} \frac{c_{11} c_9^2}{\ln\bigl(2+16d^2\bigr)} d Y^2 \kappa \sqrt{\ln(1+1/\kappa)}~.
\]
\end{lemma}

\vspace{0.2cm}
\noindent
We can now fix the values of the constants $c_8$ and $c_9$ and conclude the proof. Choosing $c_9$ and $c_8 \eqdef \max\bigl\{c_9/(2\sqrt{\ln 2}), c_9/(c_7 c_{10})\bigr\}$ such that $c_8 < 1/\sqrt{2}$ (condition \eqref{eqn:choice-constants2}), $c_9 < 1/4$, and $c_9 < c_{13}$, then the condition \eqref{eqn:choice-constants1} also holds, and \eqref{eqn:concentration-7} combined with Lemma~\ref{lem:chapL1-lower-technical} entails that
\[
\E\!\left[\indicator{\cA^c}\hat{L}_T\right] \leq \frac{1}{2} \frac{c_{11} c_9^2}{\ln\bigl(2+16d^2\bigr)} d Y^2 \kappa \sqrt{\ln(1+1/\kappa)}~.
\]
Substituting the last inequality in \eqref{eqn:lowerbound-standardTrick4}, we get that
\[
\E\!\left[\indicator{\cA}\left(\hat{L}_T-\inf_{\Norm{\bu}_1 \leq 1} L_T(\bu)\right)\right] \geq \frac{1}{2} \frac{c_{11} c_9^2}{\ln\bigl(2+16d^2\bigr)} d Y^2 \kappa \sqrt{\ln(1+1/\kappa)}~.
\]
By the above lower bound and the fact that, $\Prob^{\gamma,\sigma}_{\bu^*}$-almost surely, $\Norm{\bx_t}_{\infty} \leq \gamma \sqrt{2} \leq X$ for all $t=1,\ldots,T$ (since $\gamma \eqdef c_8 X$ and $c_8 \leq 1/\sqrt{2}$), we get that
\[
\sup_{\substack{\Norm{\bx_1}_{\infty},\ldots,\Norm{\bx_T}_{\infty}  \!\leq X \\ y_1,\ldots,y_T \in \R}} \left\{ \indicator{\cA}\left(\hat{L}_T-\inf_{\Norm{\bu}_1 \leq 1} L_T(\bu)\right) \right\} \geq \frac{1}{2} \frac{c_{11} c_9^2}{\ln\bigl(2+16d^2\bigr)} d Y^2 \kappa \sqrt{\ln(1+1/\kappa)}~.
\]
Therefore, by definition of $\cA \eqdef \bigcap_{t=1}^T \bigl\{|y_t| \leq Y \bigr\}$, of $\hat{L}_T \eqdef \sum_{t=1}^T \bigl(y_t - \tilde{f}_t(\bx_t)\bigr)^2$, and of $L_T(\bu) \eqdef \sum_{t=1}^T (y_t - \bu \cdot \bx_t)^2$, we get that, for all  online forecasters $(\tilde{f}_t)_{t \geq 1}$ whose predictions lie in $[-Y,Y]$,
\begin{align*}
\sup_{\substack{\Norm{\bx_1}_{\infty},\ldots,\Norm{\bx_T}_{\infty}  \!\leq X \\ |y_1|,\ldots,|y_T| \leq Y}} \Biggl\{ \sum_{t=1}^T \bigl(y_t - \tilde{f}_t(\bx_t)\bigr)^2 - \inf_{\Norm{\bu}_1 \leq U} \sum_{t=1}^T (y_t - \bu \cdot \bx_t)^2 \Biggr\} & \geq \frac{1}{2} \frac{c_{11} c_9^2}{\ln\bigl(2+16d^2\bigr)} d Y^2 \kappa \sqrt{\ln(1+1/\kappa)}~.
\end{align*}
Combining the last lower bound with \eqref{eqn:lowerbound-clipping} and setting $c_1 \eqdef c_{11} c_9^2/2$ concludes the proof under the assumption $\sqrt{\ln(1+2d)}/\bigl(2 d \sqrt{\ln 2}\bigr) \leq \kappa \leq 1$.\\

\noindent
\underline{Assume now that $\kappa > 1$}. \\
	The stated lower bound follows from the case when $\kappa = 1$ and by monotonicity of the minimax regret in $\kappa$ (when $d$ and $Y$ are kept constant). \\
\ \\
More formally, by the first part of this proof (when $\kappa=1$), we can fix $T \geq 1$, $U_1>0$, and $X>0$ such that $\sqrt{T}U_1 X/(2dY) = 1$ and
\begin{align*}
\inf_{(\tilde{f}_t)_t} \sup_{\substack{\Norm{\bx_t}_{\infty} \!\leq X \\ \Abs{y_t}\leq Y}} \Biggl\{ \sum_{t=1}^T \bigl(y_t - \tilde{f}_t(\bx_t)\bigr)^2 - \inf_{\Norm{\bu}_1 \leq U_1} \sum_{t=1}^T (y_t - \bu \cdot \bx_t)^2 \Biggr\} 
& \geq \frac{c_1}{\ln\bigl(2+16 d^2\bigr)} d Y^2 \sqrt{\ln 2}~,
\end{align*}
where the infimum is taken over all online forecasters $(\tilde{f}_t)_{t \geq 1}$, and where the supremum is taken over all individual sequences bounded by $X$ and~$Y$. \\
\ \\
Now take $\kappa > 1$, and set $U \eqdef \kappa U_1 > U_1$, so that $\sqrt{T}UX/(2dY) = \kappa$ (since ${\sqrt{T}U_1 X/(2dY) = 1}$). Moreover, for all individual sequences bounded by $X$ and $Y$, the regret on $B_1(U)$ is at least as large as the regret on $B_1(U_1)$ (since $U > U_1$). Combining the latter remark with the lower bound above and setting $c_2 \eqdef c_1 \sqrt{\ln 2}$ concludes the proof.
\end{proofref}

\vspace{0.3cm}
\begin{proofref}{Lemma~\ref{lem:chapL1-lower-Batchconversion}}
\label{proof:chapL1-lower-Batchconversion}
We use the same notations as in Step~1 of the proof of Theorem~\ref{thm:lowerbound}.
Let $(X',y')$ be a random copy of $(X_1,y_1)$ independent of the sample $(X_t,y_t)_{1 \leq t \leq T}$, and define the random vector $\bx' \eqdef \bigl(\gamma \varphi_1(X'), \ldots, \gamma \varphi_d(X')\bigr)$. By the tower rule, we have
\[
\E\bigl[(y_t-\tilde{f}_t(\bx_t)^2\bigr]=\E\Bigl[\E\bigl[(y_t-\tilde{f}_t(\bx_t))^2 \big| (\bx_s,y_s)_{s \leq t-1}\bigr] \Bigr] = \E\bigl[(y'-\tilde{f}_t(\bx')^2\bigr]~,
\]
where we used the fact that $\tilde{f}_t$ is built on the past data $(\bx_s,y_s)_{s \leq t-1}$ and that $(\bx',y')$ and $(\bx_t,y_t)$ are both independent of $(\bx_s,y_s)_{s \leq t-1}$ and are identically distributed. Similarly $\E\bigl[(y_t-\bu \cdot \bx_t)^2\bigr] = \E\bigl[(y'-\bu \cdot \bx')^2\bigr]$. Using the last equalities and the fact that $\E\bigl[\inf \{\ldots\}\bigr] \leq \inf \E\bigl[\{\ldots\}\bigr]$, we get
\begin{align}
& \E\!\left[\sum_{t=1}^T \bigl(y_t-\tilde{f}_t(\bx_t)\bigr)^2 - \inf_{\Norm{\bu}_1 \leq 1} \sum_{t=1}^T \bigl(y_t-\bu \cdot \bx_t\bigr)^2 \right] \nonumber \\
& \quad \geq T \left(\frac{1}{T} \sum_{t=1}^T \E\!\left[\bigl(y'-\tilde{f}_t(\bx')\bigr)^2\right] - \inf_{\Norm{\bu}_1 \leq 1} \E\!\left[\bigl(y'-\bu \cdot \bx'\bigr)^2 \right] \right) \nonumber \\
& \quad \geq T \left( \E\!\left[\bigl(y'-\hat{f}_T(X')\bigr)^2\right] - \inf_{\Norm{\bu}_1 \leq 1} \E\!\left[\bigl(y'-\bu \cdot \bx'\bigr)^2 \right] \right)  \label{eqn:lowerbound-standardTrick1}\\
& \quad = T \, \E\!\left[\bigl(\gamma \varphi_{\bu^*}(X') - \hat{f}_T(X')\bigr)^2\right] \label{eqn:lowerbound-standardTrick2} \\
& \quad = T \, \E\Norm{\hat{f}_T - \gamma \varphi_{\bu^*}}^2_{\mu}~. \nonumber
\end{align}
Inequality \eqref{eqn:lowerbound-standardTrick1} follows by definition of $\hat{f}_T \eqdef T^{-1}\sum_{t=1}^T \tilde{f}_t$ (see \eqref{eqn:lowerbound-standardTrick}) and by Jensen's inequality. As for Inequality \eqref{eqn:lowerbound-standardTrick2}, it follows by expanding the square
\[
\bigl(y'-\hat{f}_T(X')\bigr)^2 = \bigl(\gamma \varphi_{\bu^*}(X') - \hat{f}_T(X') + y'-\gamma \varphi_{\bu^*}(X')\bigr)^2~,
\]
by noting that $\E\bigl[ y'-\gamma \varphi_{\bu^*}(X') \big| X' \bigr]= 0$ (via \eqref{eqn:chapL1-lowerbound-model-reduction}) and by the fact that
\[
\inf_{\Norm{\bu}_1 \leq 1} \E\!\left[\bigl(y'-\bu \cdot \bx'\bigr)^2 \right] = \E\!\left[\bigl(y'-\gamma \varphi_{\bu^*}(X')\bigr)^2 \right]~,
\]
where we used $\Norm{\bu^*}_1 \leq 1$ (by definition of $\bu^*$) and $\bu \cdot \bx' = \gamma \varphi_{\bu}(X')$. This concludes the proof.
\end{proofref}

\vspace{0.3cm}
\begin{proofref}{Lemma~\ref{lem:chapL1-lower-technical}}
\label{proof:chapL1-lower-technical}
We use the same notations and assumptions as in the proof of Theorem~\ref{thm:lowerbound}.
Since the function $x \mapsto x \sqrt{\ln(1+1/x)}$ is nondecreasing on $\R_+^*$ and since $\kappa \geq \kappa_{\min} \eqdef \sqrt{\ln(1+2d)}/(2d \sqrt{\ln 2})$ by assumption, we have
\begin{align}
& \frac{c_{11} c_9^2}{\ln\bigl(2+16d^2\bigr)} d Y^2 \kappa \sqrt{\ln(1+1/\kappa)} \nonumber \\
& \quad \geq \frac{c_{11} c_9^2}{\ln\bigl(2+16d^2\bigr)} d Y^2 \kappa_{\min} \sqrt{\ln(1+1/\kappa_{\min})} \nonumber \\
& \quad = \frac{c_{11} c_9^2}{2 \sqrt{\ln 2}} Y^2 \frac{\sqrt{\ln(1+2d)}\sqrt{\ln\Bigl[1+2d \sqrt{\ln 2}/\sqrt{\ln(1+2d)}\Bigr]}}{\ln\bigl(2+16d^2\bigr)} \label{eqn:chapL1-lowerbound-technical-0} \\
& \quad \geq \frac{c_{11} c_9^2}{2 \sqrt{\ln 2}} Y^2 c_{12}~, \label{eqn:chapL1-lowerbound-technical-1}
\end{align}
where $c_{12}$ denotes the infimum of the last fraction of \eqref{eqn:chapL1-lowerbound-technical-0} over all $d \geq 1$; in particular, $c_{12} > 0$. It is now easy to see that by choosing the absolute constant $c_{13}>0$ small enough (where $c_{13}$ can be expressed in terms of $c_{11}$ and $c_{12}$), we have, for all $c_9 \in (0,c_{13})$,
\[
8 \cdot 2^{2-1/(8 c_9^2)} + 9 \cdot 2^{1-1/(8 c_9^2)} + \frac{2 c_9^2 \sqrt{6}}{\ln 2} 2^{1-1/(16 c_9^2)} \leq \frac{c_{11} c_9^2}{2 \sqrt{\ln 2}} c_{12}~.
\]
Multiplying both sides of the last inequality by $Y^2$ and combining it with \eqref{eqn:chapL1-lowerbound-technical-1} concludes the proof.
\end{proofref}

\subsection{Proofs of Theorem~\ref{thm:ell1-EG-regret} and Remark~\ref{rem:chapL1-simpler-analysis}}
\label{sec:chapL1-proof-EG-regret}

\begin{proofref}{Theorem~\ref{thm:ell1-EG-regret}}
The proof follows directly from Proposition~\ref{prop:adaptiveEGpm} and from the fact that the Lipschitzified losses are larger than their clipped versions. Indeed, first note that, by definition of $\hat{y}_t$ and $B_{t+1} \geq |y_t|$, we have
  \begin{align}
    \sum_{t=1}^T \Abs{y_t - \hat y_t}^{\alpha} & \leq \sum_{\substack{t=1 \\ t:|y_t|\leq B_t}}^T \Abs{y_t - \bigl[\hat \bu_t \cdot \bx_t\bigr]_{B_t}}^{\alpha}  + \sum_{\substack{t=1 \\ t:|y_t| > B_t}}^T (B_{t+1}+B_t)^{\alpha} \nonumber \\
    & \leq \sum_{\substack{t=1 \\ t:|y_t|\leq B_t}}^T \tilde{\ell}_t(\hat \bu_t) + \left(1+2^{-1/\alpha}\right)^{\alpha} \!\sum_{\substack{t=1 \\ t:B_{t+1} > B_t}}^T B_{t+1}^{\alpha} \nonumber  \\
    & \leq \sum_{t=1}^T \tilde{\ell}_t(\hat{\bu}_t) + 4
    \left(1+2^{-1/\alpha}\right)^{\alpha} Y^{\alpha}~, \label{eqn:chapL1-ISL-0}
  \end{align}
  where the second inequality follows from the fact that:
  \begin{itemize}
  \item if $|y_t| \leq B_t$ then $\Abs{y_t - [\hat{\bu}_t \cdot
    \bx_t]_{B_t}}^{\alpha} \leq \tilde{\ell}_t(\hat{\bu}_t)$ by
    Eq.\ (\ref{eqn:Lip-properties});
  \item if $|y_t| > B_t$, which is equivalent to $B_{t+1} > B_t$ by
    definition of $B_{t+1}$, then $B_t \leq B_{t+1}/2^{1/\alpha}$, so that
    $B_{t+1} + B_t \leq \left(1+2^{-1/\alpha}\right) B_{t+1}$.
  \end{itemize}
  As for the third inequality above, we used the non-negativity of
  $\tilde{\ell}_t(\hat{\bu}_t)$ and upper bounded the geometric sum
  $\sum_{t:B_{t+1} > B_t}^T B_{t+1}^{\alpha}$ in the same way as in
  \cite[Theorem~6]{CeMaSt07SecOrder}, \ie, setting $K \eqdef \bigl\lceil
  \log_2 \max_{1 \leq t \leq T} \Abs{y_t}^{\alpha} \bigr\rceil$,
  \[
  \sum_{t:B_{t+1} > B_t}^T B_{t+1}^{\alpha} \leq \sum_{k=-\infty}^K 2^k =
  2^{K+1} \leq 4 Y^{\alpha}~.
  \]
  To bound \eqref{eqn:chapL1-ISL-0} further from above, we now use the fact that, by construction, the LEG algorithm is the adaptive $\textrm{EG}^{\pm}$ algorithm applied to the modified loss functions $\tilde{\ell}_t$. Therefore, we get from Proposition~\ref{prop:adaptiveEGpm} that
\begin{align}
\sum_{t=1}^T \tilde{\ell}_t(\hat{\bu}_t) & \leq \inf_{\Norm{\bu}_1 \leq U} \sum_{t=1}^T \tilde{\ell}_t(\bu) \nonumber \\
& \quad + 4 U \sqrt{\left(\sum_{t=1}^T \norm[\nabla \tilde{\ell}_t(\hat{\bu}_t)]_{\infty}^2\right) \ln(2d)} + U \, \bigl(8 \ln(2d) + 12 \bigr) \max_{1 \leq t \leq T} \norm[\nabla \tilde{\ell}_t(\hat{\bu}_t)]_{\infty}~. \label{eqn:chapL1-ISL-1}
  \end{align}
We can now follow the same lines as in Corollary~\ref{cor:introM-EGpm-square}, except that we use the particular shape of the Lipschitzified losses. We first derive some properties of the gradients $\nabla \tilde{\ell}_t$. Observe from the definition of $\tilde{\ell}_t$ in Section~\ref{sec:loss-Lip} that in both cases $|y_t| > B_t$ and $|y_t| \leq B_t$, the function $\widetilde{\ell}_t$ is continuously differentiable. Moreover, if $|y_t| \leq B_t$, then
\begin{equation*}
\forall \bu \in \R^d~, \quad \nabla \tilde{\ell}_t(\bu) = -\alpha \, \sgn\bigl(y_t - [\bu \cdot \bx_t]_{B_t} \bigr)  \Abs{y_t - [\bu \cdot \bx_t]_{B_t}}^{\alpha-1} \bx_t~,
\end{equation*}
where for all $x \in \R$, the quantity $\sgn(x)$ equals $1$ (resp.\ $-1$, $0$) if $x>0$ (resp.\ $x<0$, $x=0$).\\

\noindent
Therefore, in both cases $|y_t| > B_t$ and $|y_t| \leq B_t$, the function $\widetilde{\ell}_t$ is Lipschitz continuous with respect to
$\Norm{\cdot}_1$ with Lipschitz constant $\sup_{\bu \in \R^d} \Norm{\nabla \tilde{\ell}_t}_{\infty}$ bounded as follows: for all $\bu \in \R^d$,
\begin{align}
	\Norm{\nabla \tilde{\ell}_t(\bu)}_{\infty} & \leq \alpha \, \Abs{y_t - [\bu \cdot \bx_t]_{B_t}}^{\alpha-1} \, \Norm{\bx_t}_{\infty} \label{eqn:chapL1-Lip-gradients} \\
  & \leq \alpha \, \bigl(|y_t|+B_t\bigr)^{\alpha-1} \Norm{\bx_t}_{\infty} \leq \alpha \,  \bigl(1+2^{1/\alpha}\bigr)^{\alpha-1} \left(\max_{1 \leq s \leq t} |y_s|\right)^{\alpha-1} \Norm{\bx_t}_{\infty}~,\label{eqn:Lip-constant}
\end{align}
where we used the fact that $B_t \leq 2^{1/\alpha} \max_{1 \leq s \leq t-1} |y_s|$.\\

We can draw several consequences from the inequalities above. First note that, by \eqref{eqn:Lip-constant},
\begin{equation}
\label{eqn:adaptiveEG-lip-1}
\max_{1 \leq t \leq T} \Arrowvert\nabla \tilde{\ell}_t(\hat{\bu}_t)\Arrowvert_{\infty} \leq \alpha \bigl(1+2^{1/\alpha}\bigr)^{\alpha-1} X Y^{\alpha-1}~.
\end{equation}
Moreover, using \eqref{eqn:chapL1-Lip-gradients} and the definition of $\hat{y}_t$ in Figure~\ref{fig:algo-LEG}, we can see that the gradients $\nabla \tilde{\ell}_t(\hat{\bu}_t)$ satisfy $\norm[\nabla \tilde{\ell}_t(\hat{\bu}_t)]_{\infty} \leq \alpha \, \Abs{y_t - \hat{y}_t}^{\alpha-1} \, \Norm{\bx_t}_{\infty} \leq \alpha X \Abs{y_t - \hat{y}_t}^{\alpha-1}$. This entails that
\begin{align}
\norm[\nabla \tilde{\ell}_t(\hat{\bu}_t)]_{\infty}^2 & \leq \alpha^2 X^2 \big|y_t - \hat{y}_t\big|^{2\alpha-2} = \alpha^2 X^2 \big|y_t -\hat{y}_t\big|^{\alpha-2} \, \big|y_t - \hat{y}_t\big|^{\alpha} \nonumber \\
& \leq \alpha^2 X^2 \bigl( (1+2^{1/\alpha}) Y \bigr)^{\alpha-2} \, \big|y_t - \hat{y}_t\big|^{\alpha}~, \label{eqn:adaptiveEG-lip-2}
\end{align}
where we used the upper bounds $\Abs{y_t} \leq Y$ and $\Abs{\hat{y}_t} \eqdef \Abs{\bigl[\hat{\bu}_t \cdot \bx_t\bigr]_{B_t}} \leq B_t \leq 2^{1/\alpha} Y$.
Substituting \eqref{eqn:adaptiveEG-lip-1} and~\eqref{eqn:adaptiveEG-lip-2} in~\eqref{eqn:chapL1-ISL-1} and combining the resulting bound with \eqref{eqn:chapL1-ISL-0}, we get
\begin{align*}
\sum_{t=1}^T \Abs{y_t - \hat y_t}^{\alpha} & \leq \inf_{\Norm{\bu}_1 \leq U} \sum_{t=1}^T \tilde{\ell}_t(\bu) + a_{\alpha} U X Y^{\alpha/2-1} \, \sqrt{\left( \sum_{t=1}^T \Abs{y_t - \hat y_t}^{\alpha} \right)\ln(2d)} \\[0.1cm]
& \qquad + \, \underbrace{\bigl(8 \ln(2d)+12 \bigr) \, b_{\alpha} \, U X Y^{\alpha-1}}_{\eqdef \, C_1} + \, \underbrace{4\bigl(1+2^{-1/\alpha}\bigr)^{\alpha} Y^{\alpha}}_{ \eqdef \, C_2}~,
\end{align*}
where we set $a_{\alpha} \eqdef 4 \alpha \, \bigl(1+2^{1/\alpha}\bigr)^{\alpha/2-1}$ and $b_{\alpha} \eqdef \alpha \, \bigl(1+2^{1/\alpha}\bigr)^{\alpha-1}$.

To simplify the notations we also set $\hat{L}_T \eqdef \sum_{t=1}^T \Abs{y_t - \hat y_t}^{\alpha}$ and $\tilde{L}_T^* \eqdef \min_{\norm[\bu]_1 \leq U} \sum_{t=1}^T \tilde{\ell}_t(\bu)$, so that the previous inequality can be rewritten as
\[
\hat{L}_T \leq \tilde{L}_T^* + C_1 + C_2 + a_{\alpha} U X Y^{\alpha/2-1} \, \sqrt{\hat{L}_T \ln(2d)}~.
\]

\noindent
Solving for $\hat{L}_T$ via Lemma~\ref{lem:introM-solvingRegret} in \refapx{sec:chapL1-lemmas} (used with $a=\tilde{L}_T^* + C_1 + C_2$ and $b=a_{\alpha} U X Y^{\alpha/2-1} \sqrt{\ln(2d)}$), we get that
\begin{align}
\hat{L}_T & \leq \tilde{L}_T^* + C_1 + C_2 + \left( a_{\alpha} U X Y^{\alpha/2-1} \, \sqrt{\ln(2d)} \right) \sqrt{\tilde{L}_T^* + C_1 + C_2} + \left(a_{\alpha} U X Y^{\alpha/2-1} \, \sqrt{\ln(2d)} \right)^2 \nonumber \\
& \leq \tilde{L}_T^* + a_{\alpha} U X Y^{\alpha/2-1} \, \sqrt{\tilde{L}_T^* \ln(2d)} \nonumber \\
& \qquad + a_{\alpha} U X Y^{\alpha/2-1} \, \sqrt{(C_1 + C_2) \ln(2d)} + a_{\alpha}^2 U^2 X^2 Y^{\alpha-2} \ln(2d) + C_1 + C_2~.\label{eqn:chapL1-ISL-2}
\end{align}
To conclude the proof, it just suffices to bound the term $a_{\alpha} U X Y^{\alpha/2-1} \, \sqrt{(C_1 + C_2) \ln(2d)}$ from above. First note that
\begin{align}
\sqrt{(C_1 + C_2) \ln(2d)} & \leq \sqrt{C_1 \ln(2d)} + \sqrt{C_2 \ln(2d)} \nonumber \\
& \leq \sqrt{C_1 \ln(2d)} + 2 \bigl(1+2^{-1/\alpha}\bigr)^{\alpha/2} Y^{\alpha/2} \sqrt{\ln(2d)}~, \label{eqn:chapL1-ISL-3}
\end{align}
where the last inequality follows by definition of $C_2$ above. Now, to upper bound $\sqrt{C_1 \ln(2d)}$, we note that, by definition of $C_1$,
\begin{align*}
\sqrt{C_1 \ln(2d)} & = \ln(2d) \, \sqrt{\bigl(8 + 12/\ln(2d) \bigr) \, b_{\alpha} \, U X Y^{\alpha-1}} \\
& \leq \ln(2d) \, \sqrt{\bigl(8 + 12/\ln 2 \bigr) \, b_{\alpha}} \, \, \frac{UXY^{\alpha/2-1} + Y^{\alpha/2}}{\sqrt{2}}~,
\end{align*}
where we used the elementary upper bound $\sqrt{a b} \leq (a+b)/2$ with $a=U X Y^{\alpha/2-1}$ and $b=Y^{\alpha/2}$. Substituting the last inequality in~\eqref{eqn:chapL1-ISL-3} and using $\sqrt{\ln(2d)} \leq \ln(2d)/\sqrt{\ln 2}$, we finally get that
\begin{align*}
& a_{\alpha} U X Y^{\alpha/2-1} \, \sqrt{(C_1 + C_2) \ln(2d)} \\
& \qquad \leq a_{\alpha} \ln(2d) \left( \sqrt{b_{\alpha} \bigl(4+6/\ln 2 \bigr)} + 2 \bigl(1+2^{-1/\alpha}\bigr)^{\alpha/2} / \sqrt{\ln 2} \right) UXY^{\alpha-1} \\
& \qquad \qquad + a_{\alpha} \ln(2d) \, \sqrt{b_{\alpha} \bigl(4+6/\ln 2 \bigr)} \; U^2 X^2 Y^{\alpha-2}~.
\end{align*}
Substituting the last inequality into \eqref{eqn:chapL1-ISL-2} and rearranging terms concludes the proof.
\end{proofref}

\vspace{0.3cm}
\begin{proofref}{Remark~\ref{rem:chapL1-simpler-analysis}}
Recall that in this remark, we focus on the square loss (\ie, $\alpha=2$) and that we set $c_1 \eqdef 8 \bigl(\sqrt{2}+1\bigr)$ and $c_2 \eqdef 4\left(1+1/\sqrt{2}\right)^2$. By the key property \eqref{eqn:Lip-properties} that holds for all rounds $t$ such that $|y_t| \leq B_t$ (the other rounds accounting only for an additional total loss at most of $c_2 Y^2$, see \eqref{eqn:chapL1-ISL-0}), we get
  \begin{align}
    \sum_{t=1}^T (y_t - \hat y_t)^2 - \inf_{\Norm{\bu}_1 \leq U} \sum_{t=1}^T (y_t - \bu \cdot \bx_t)^2 & \leq  \sum_{t=1}^T \tilde{\ell}_t(\hat{\bu}_t) - \inf_{\Norm{\bu}_1 \leq U} \sum_{t=1}^T  \tilde{\ell}_t(\bu) + c_2 Y^2 \nonumber \\
    & \leq 4 U \max_{1 \leq t \leq T} \Norm{\nabla
      \tilde{\ell}_t(\hat{\bu}_t)}_{\infty} \left(\sqrt{T \ln (2d)} + 2 \ln
      (2d) + 3 \right) + c_2 Y^2 \label{eqn:chapL1-simpler-analysis-1} \\
    &  \leq c_1 U X Y \left(\sqrt{T \ln (2 d)} + 8 \ln (2
      d)\right) + c_2 Y^2~, \label{eqn:chapL1-simpler-analysis-2}
  \end{align}
  where \eqref{eqn:chapL1-simpler-analysis-1} follows from the remark in Proposition~\ref{prop:adaptiveEGpm} involving the uniform bound $\max_{1 \leq t \leq T} \Arrowvert \nabla \tilde{\ell}_t(\hat{\bu}_t) \Arrowvert_{\infty}$, and where \eqref{eqn:chapL1-simpler-analysis-2} follows from $\max_{1 \leq t \leq T} \Arrowvert \nabla \tilde{\ell}_t(\hat{\bu}_t) \Arrowvert_{\infty} \leq 2 \bigl(1+\sqrt{2}\bigr) X Y$ (by \eqref{eqn:Lip-constant}) and from the elementary inequality $3 \leq 6 \ln (2d)$.
\end{proofref}

%%%%%%%%%%%%%%%%%%%%%%%%%%%%%%%%%%%%%%%%%%%%%%%%%%%%%%%%%%%%%%%%%%%%%
\section{Lemmas}
\label{sec:chapL1-lemmas}
%%%%%%%%%%%%%%%%%%%%%%%%%%%%%%%%%%%%%%%%%%%%%%%%%%%%%%%%%%%%%%%%%%%%%

The next elementary lemma is due to \cite[Appendix~III]{CeLuSt-05-LabelEfficient}. It is useful to compute an upper bound on the cumulative loss $\hat{L}_T$ of a forecaster when $\hat{L}_T$ satisfies an inequality of the form~\eqref{eqn:solvingRegret}.

\begin{lemma}
\label{lem:introM-solvingRegret}
Let $a,b \geq 0$. Assume that $x \geq 0$ satisfies the inequality
\begin{equation}
\label{eqn:solvingRegret}
x \leq a + b \sqrt{x}~.
\end{equation}
Then,
\[
x \leq a + b \sqrt{a} + b^2~.
\]
\end{lemma}

\vspace{0.2cm}
The next lemma is useful to prove Theorem~\ref{thm:upperbound}. At the end of this section, we also provide an elementary lemma about the exponentially weighted average forecaster combined with clipping.

\vspace{0.2cm}
\begin{lemma}
\label{lem:chapL1-upper-lemma}
Let $d, T \in \Zp$, and $U, X, Y >0$. The minimax regret on $B_1(U)$ for bounded base predictions and observations satisfies
  \begin{align*}
    & \inf_F \sup_{\Norm{\bx_t}_\infty \leq X,\; \Abs{y_t}\leq Y} \Biggl\{ \sum_{t=1}^T (y_t - \hat y_t)^2 - \inf_{\Norm{\bu}_1 \leq U} \sum_{t=1}^T (y_t - \bu \cdot \bx_t)^2 \Biggr\} \\
    & \qquad \leq \min\left\{3UXY\sqrt{2 T \ln (2d)}, \, 32 \, d Y^2 \ln\!\left(1+\frac{\sqrt{T} U X}{d Y}\right) + d Y^2 \right\}~,
\end{align*}
  where the infimum is taken over all forecasters $F$ and where the
  supremum extends over all sequences $(\bx_t,y_t)_{1 \leq t \leq T}
  \in (\R^d \times \R)^T$ such that $|y_1|, \ldots, |y_T| \leq Y$ and
  $\Norm{\bx_1}_{\infty}, \ldots, \Norm{\bx_T}_{\infty} \leq~X$.
\end{lemma}

\vspace{0.4cm}
\begin{proof}
We treat each of the two terms in the above minimum separately.\\

\noindent
{\bf Step 1}: We prove that their exists a forecaster $F$ whose worst-case regret on $B_1(U)$ is upper bounded by $3UXY\sqrt{2 T \ln (2d)}$.\\

First note that if $U \geq (Y/X) \sqrt{T/(2 \ln(2d))}$, then the upper bound $3UXY\sqrt{2 T \ln (2d)} \geq 3 T Y^2 \geq T Y^2$ is trivial (by choosing the forecaster $F$ which outputs $\hat{y}_t=0$ at each time $t$). \\
	
We can thus assume that $U < (Y/X) \sqrt{T/(2 \ln(2d))}$. Consider the $\textrm{EG}^{\pm}$ algorithm as given in \cite[Theorem~5.11]{KiWa97EGvsGD}, and denote by $\hat{\bu}_t \in B_1(U)$ the linear combination it outputs at each time $t \geq 1$. Then, by the aforementioned theorem, this forecaster satisfies, uniformly over all individual sequences bounded by $X$ and $Y$, that
	\begin{align}
	& \sum_{t=1}^T (y_t - \hat{\bu}_t \cdot \bx_t)^2 - \inf_{\Norm{\bu}_1 \leq U} \sum_{t=1}^T (y_t - \bu \cdot \bx_t)^2 \nonumber \\
	& \quad \leq 2UXY\sqrt{2 T \ln(2d)} + 2 U^2 X^2 \ln(2d) \nonumber \\
	& \quad \leq 2UXY\sqrt{2 T \ln(2d)} + 2 \left(Y \sqrt{\frac{T}{2 \ln(2d)}} \, \right) U X \ln(2d) \label{eqn:EG-upperbound} \\
	& \quad \leq 3 U X Y \sqrt{2 T \ln(2d)}~, \nonumber
\end{align}
where (\ref{eqn:EG-upperbound}) follows from the assumption $U X < Y \sqrt{T/(2 \ln(2d))}$. This concludes the first step of this proof. \\

\noindent
{\bf Step 2}: We prove that their exists a forecaster $F$ whose worst-case regret on $B_1(U)$ is upper bounded by $32 \, d Y^2 \ln\!\left(1+\frac{\sqrt{T} U X}{d Y}\right) + d Y^2$.

  Such a forecaster is given by the sparsity-oriented algorithm $\textrm{SeqSEW}^{B,\eta}_{\tau}$ of \cite{Ger-11colt-SparsityRegretBounds} (we could also get a slightly worse bound with the sequential ridge regression forecaster of \cite{AzWa01RelativeLossBounds,Vo01CompetitiveOnline}). Indeed, by \cite[Proposition~1]{Ger-11colt-SparsityRegretBounds}, the cumulative square loss of the algorithm $\textrm{SeqSEW}^{B,\eta}_{\tau}$ tuned with $B=Y$, $\eta=1/(8 Y^2)$ and
  $\tau=Y/(\sqrt{T} X)$ is upper bounded by
  \begin{align*}
    & \inf_{\bu \in \R^d} \left\{\sum_{t=1}^T \bigl(y_t - \bu \cdot \bx_t\bigr)^2 + 32 \Norm{\bu}_0 Y^2 \ln\!\left(1 + \frac{\sqrt{T} X \Norm{\bu}_1}{\Norm{\bu}_0 Y}\right)\right\} + d Y^2 \\
    & ~\leq \inf_{\Norm{\bu}_1 \leq U} \left\{\sum_{t=1}^T \bigl(y_t -
      \bu \cdot \bx_t\bigr)^2 \right\} + 32 d Y^2 \ln\!\left(1 +
      \frac{\sqrt{T} X U}{d Y} \right) + d Y^2~,
  \end{align*}
  where the last inequality follows by monotonicity\footnote{Note that
    for all $A>0$, the function $x \mapsto x \ln(1+A/x)$ (continuously
    extended at $x=0$) has a nonnegative first derivative and is thus
    nondecreasing on $\R_+$.} in $\Norm{\bu}_0$ and $\Norm{\bu}_1$ of
  the second term of the left-hand side. This concludes the proof.
\end{proof}

\vspace{0.5cm}
Next we recall a regret bound satisfied by the standard exponentially weighted average forecaster applied to clipped base forecasts. Assume that at each time~$t\geq 1$, the forecaster has access to $K \geq 1$ base forecasts $\hat{y}^{(k)}_t \in \R$,
$k=1,\ldots,K$, and that for some known bound $Y>0$ on the observations, the forecaster predicts at time~t as
\[
\hat{y}_t \eqdef \sum_{k=1}^K p_{k,t} \bigl[\hat{y}^{(k)}_t\bigr]_Y~.
\]
In the equation above, $[x]_Y \eqdef \min\{Y,\max\{-Y,x\}\}$ for all $x \in \R$, and the weight vectors $\bp_t \in \R^K$ are given by $\bp_1=(1/K, \ldots, 1/K)$ and, for all $t=2,\ldots,T$, by
\[
p_{k,t} \eqdef \frac{\exp\left(-\eta \sum_{s=1}^{t-1}
    \left(y_s-\bigl[\hat{y}^{(k)}_s\bigr]_Y\right)^2\right)}{\sum_{j=1}^K
  \exp\left(-\eta \sum_{s=1}^{t-1}
    \left(y_s-\bigl[\hat{y}^{(j)}_s\bigr]_Y\right)^2\right)}~, \quad 1
\leq k \leq K~,
\]
for some inverse temperature parameter $\eta > 0$ to be chosen below. The next lemma is a straigthforward consequence of Theorem~3.2 and Proposition~3.1 of \cite{cesa-bianchi06prediction}.

\vspace{0.1cm}
\begin{lemma}[Exponential weighting with clipping]
  \label{lem:EWA-exp-concave}
  Assume that the forecaster knows beforehand a bound $Y>0$ on the
  observations $|y_t|$, $t=1,\ldots,T$.
  Then, the exponentially weighted average forecaster tuned with $\eta
  \leq 1/(8Y^2)$ and with clipping $[\, \cdot \,]_Y$ satisfies
  \[
  \sum_{t=1}^T \bigl(y_t - \hat{y}_t\bigr)^2 \leq \min_{1 \leq k \leq
    K} \sum_{t=1}^T \bigl(y_t - \hat{y}^{(k)}_t\bigr)^2 + \frac{\ln
    K}{\eta}~.
  \]
\end{lemma}

\begin{proofref}{Lemma~\ref{lem:EWA-exp-concave}}
The proof follows straightforwardly from Theorem~3.2 and Proposition~3.1 of \cite{cesa-bianchi06prediction}. To apply the latter result, recall from  \cite[Remark~3]{Vo01CompetitiveOnline} that the square loss is $1/(8 Y^2)$-exp-concave on $[-Y, Y]$ and thus $\eta$-exp-concave\footnote{This means that for all $y \in [-Y,Y]$, the function $x \mapsto \exp\bigl(- \eta (y-x)^2\bigr)$ is concave on $[-Y, Y]$.} (since $\eta \leq 1/(8 Y^2)$ by assumption). Therefore, by definition of our forecaster above, Theorem~3.2 and Proposition~3.1 of \cite{cesa-bianchi06prediction} yield
  \[
  \sum_{t=1}^T \bigl(y_t - \hat{y}_t\bigr)^2 \leq \min_{1 \leq k \leq
    K} \sum_{t=1}^{T} \left(y_t-\bigl[\hat{y}^{(k)}_t\bigr]_Y\right)^2
  + \frac{\ln K}{\eta}~.
  \]
  To conclude the proof, note for all $t=1,\ldots,T$ and
  $k=1,\ldots,K$ that $|y_t| \leq Y$ by assumption, so that clipping
  the base forecasts to $[-Y,Y]$ can only improve prediction, \ie,
  $\bigl(y_t-\bigl[\hat{y}^{(k)}_t\bigr]_Y\bigr)^2 \leq
  \bigl(y_t-\hat{y}^{(k)}_t\bigr)^2$.
\end{proofref}

%%%%%%%%%%%%%%%%%%%%%%%%%%%%%%%%%%%%%%%%%%%%%%%%%%%%%%%%%%%%%%%%%%%%%
\section{Additional tools}
\label{apx:additional}
%%%%%%%%%%%%%%%%%%%%%%%%%%%%%%%%%%%%%%%%%%%%%%%%%%%%%%%%%%%%%%%%%%%%%

The next approximation argument is originally due to Maurey, and was
used under various forms, \eg, in \cite{Nem-00-TopicsNonparametric,Tsy-03-OptimalRates,BuNo08SeqProcedures,ShSrZh-10-Sparsifiability} (see also \cite{Yan-04-BetterPerformance}).

\vspace{0.2cm}
\begin{lemma}[Approximation argument]
  \label{lem:Maurey-approx}
  Let $U > 0$ and $m \in \N^*$. Define the following finite subset of
  $B_1(U)$:
  \[
  \tilde{B}_{U,m} \eqdef \left\{ \left(\frac{k_1 U}{m}, \ldots,
      \frac{k_d U}{m}\right): (k_1, \ldots, k_d) \in \Z^d,
    \sum_{j=1}^d |k_j| \leq m \right\} \subset B_1(U)~.
  \]
  Then, for all $(\bx_t,y_t)_{1 \leq t \leq T} \in \bigl(\R^d \times
  \R \bigr)^T$ such that $ \max_{1 \leq t \leq T}
  \Norm{\bx_t}_{\infty} \leq X$,
  \[
  \inf_{\bu \in \tilde{B}_{U,m}} \sum_{t=1}^T (y_t - \bu \cdot
  \bx_t)^2 \leq \inf_{\bu \in B_1(U)} \sum_{t=1}^T (y_t - \bu \cdot
  \bx_t)^2 + \frac{T U^2 X^2}{m}~.
  \]
\end{lemma}

\begin{proof}
  The proof is quite standard and follows the same lines as \cite[Proposition~5.2.2]{Nem-00-TopicsNonparametric} or \cite[Theorem~2]{BuNo08SeqProcedures} who addressed the aggregation task in the stochastic setting. We rewrite this argument below in our online deterministic setting. \\
  \ \\
  Fix $\bu^* \in \argmin_{\bu \in B_1(U)} \sum_{t=1}^T (y_t - \bu
  \cdot \bx_t)^2$. Define the probability distribution
  $\pi=(\pi_{-d},\ldots,\pi_d) \in \R_+^{2d+1}$ by
  \begin{numcases}{\pi_j \eqdef}
    \frac{(u^*_j)_+}{U} & if $j \geq 1$; \nonumber \\
    \frac{(u^*_j)_-}{U} & if $j \leq -1$; \nonumber \\
    1 - \sum_{j=1}^d \frac{|u^*_j|}{U} & if $j=0$~. \nonumber
  \end{numcases}
  Let $J_1, \ldots, J_m \in \{-d,\ldots,d\}$ be {i.i.d.} random
  integers drawn from $\pi$, and set
  \[
  \tilde{\bu} \eqdef \frac{U}{m} \sum_{k=1}^m \be_{J_k}~,
  \]
  where $({\bf e}_j)_{1 \leq j \leq d}$ is the canonical basis of $\R^d$, where $\be_0 \eqdef {\bf 0}$, and where $\be_{-j} \eqdef - \be_j$ for all $1 \leq j \leq d$. Note that $\tilde{\bu} \in \tilde{B}_{U,m}$ by construction. Therefore,
  \begin{align}
    \inf_{\bu \in \tilde{B}_{U,m}} \sum_{t=1}^T (y_t - \bu \cdot
    \bx_t)^2 & \leq \E\!\left[\sum_{t=1}^T (y_t - \widetilde{\bu}
      \cdot \bx_t)^2\right]~. \label{eqn:maurey-lower}
  \end{align}

\noindent
  The rest of the proof is dedicated to upper bounding the last
  expectation. Expanding all the squares $(y_t-\tilde{\bu} \cdot
  \bx_t)^2 = (y_t-\bu^* \cdot \bx_t + \bu^* \cdot \bx_t - \tilde{\bu}
  \cdot \bx_t)^2$, first note that
  \begin{align}
    \E\!\left[\sum_{t=1}^T (y_t-\tilde{\bu} \cdot \bx_t)^2\right] & = \sum_{t=1}^T (y_t-\bu^* \cdot \bx_t)^2 + \sum_{t=1}^T \E\!\left[(\bu^* \cdot \bx_t - \tilde{\bu} \cdot \bx_t)^2\right] \nonumber \\
    & + 2 \sum_{t=1}^T (y_t-\bu^* \cdot \bx_t) \, \E\bigl[\bu^* \cdot
    \bx_t - \tilde{\bu} \cdot \bx_t\bigr]~. \label{eqn:maurey-upper1}
  \end{align}
  But by definition of $\tilde{\bu}$ and $\pi$,
  \begin{align*}
    \E\bigl[\tilde{\bu}\bigr] & = U \, \E\bigl[\be_{J_1}\bigr] = U \sum_{j=-d}^d \pi_j \be_j \\
    & = U \sum_{j=1}^d \left(\frac{\bigl(u^*_j\bigr)_+}{U} \be_j +
      \frac{\bigl(u^*_j\bigr)_-}{U} (-\be_j)\right) = U \sum_{j=1}^d
    \frac{u^*_j}{U} \be_j = \bu^*~,
  \end{align*}
  so that $\E\bigl[\tilde{\bu} \cdot \bx_t\bigr] = \bu^* \cdot \bx_t$
  for all $1 \leq t \leq T$. Therefore, the last sum
  in~(\ref{eqn:maurey-upper1}) above equals zero, and
  \[
  \E\!\left[\bigl(\bu^* \cdot \bx_t - \tilde{\bu} \cdot
    \bx_t\bigr)^2\right] = \Var\bigl(\tilde{\bu} \cdot \bx_t\bigr) =
  \frac{U^2}{m^2} \sum_{k=1}^m \Var\bigl(\be_{J_k} \cdot \bx_t\bigr)
  \leq \frac{U^2 X^2}{m}~,
  \]
  where the second equality follows from $\tilde{\bu} \cdot \bx_t = (U/m) \sum_{k=1}^m  \be_{J_k} \cdot \bx_t$ and from the independence of the $J_k$, $1 \leq k \leq m$, and where the last inequality follows from $|\be_{J_k} \cdot \bx_t| \leq \Norm{\be_{J_k}}_1 \Norm{\bx_t}_{\infty} \leq X$ for all $1 \leq k \leq m$. \\

\noindent
  Combining~(\ref{eqn:maurey-upper1}) with the remarks above, we get
  \begin{align*}
    \E\!\left[\sum_{t=1}^T (y_t-\tilde{\bu} \cdot \bx_t)^2\right] & \leq \sum_{t=1}^T (y_t-\bu^* \cdot \bx_t)^2 + \frac{T U^2 X^2}{m} \\
    & = \inf_{\bu \in B_1(U)} \sum_{t=1}^T (y_t - \bu \cdot \bx_t)^2 +
    \frac{T U^2 X^2}{m}~,
  \end{align*}
  where the last line follows by definition of $\bu^*$. Substituting
  the last inequality in~(\ref{eqn:maurey-lower}) concludes the proof.
\end{proof}

\ \\
The combinatorial result below (or variants of it) is well-known; see,
\eg, \cite{Tsy-03-OptimalRates,BuNo08SeqProcedures}. We reproduce its
proof for the convenience of the reader. We use the notation $\mathrm{e} \eqdef \exp(1)$.

\begin{lemma}[An elementary combinatorial upper bound]
  \label{lem:combinatorial}
  \ \\
  Let $m,d \in \N^*$. Denoting by $|E|$ the cardinality of a set $E$,
  we have
  \[
  \left| \left\{ (k_1, \ldots, k_d) \in \Z^d: \sum_{j=1}^d |k_j| \leq
      m \right\} \right| \leq
  \left(\frac{\mathrm{e}(2d+m)}{m}\right)^m~.
  \]
\end{lemma}

\begin{proofref}{Lemma~\ref{lem:combinatorial}}
  Setting $(k'_{-j},k'_j) \eqdef \bigl((k_j)_-,(k_j)_+\bigr)$ for all
  $1 \leq j \leq d$, and $k'_0 \eqdef m-\sum_{j=1}^d |k_j|$, we have
  \begin{align}
    \left| \left\{ (k_1, \ldots, k_d) \in \Z^d: \sum_{j=1}^d |k_j| \leq m \right\} \right| & \leq \left| \left\{ (k'_{-d}, \ldots, k'_d) \in \N^{2d+1}: \sum_{j=-d}^d k'_j = m \right\} \right| \nonumber \\
    & = \binom{2d+m}{m} \label{eqn:lattice-paths} \\
    & \leq
    \left(\frac{\textrm{e}(2d+m)}{m}\right)^m~. \label{eqn:sauer-type-ineq-proof}
  \end{align}
  To get inequality~(\ref{eqn:lattice-paths}), we used the
  (elementary) fact that the number of $2d+1$ integer-valued tuples
  summing up to $m$ is equal to the number of lattice paths from
  $(1,0)$ to $(2d+1,m)$ in $\N^2$, which is equal to
  $\binom{2d+1+m-1}{m}$. As for
  inequality~(\ref{eqn:sauer-type-ineq-proof}), it follows
  straightforwardly from a classical combinatorial result stated, \eg,
  in \cite[Proposition~2.5]{Massart03StFlour}.
\end{proofref}


%\bibliographystyle{elsarticle-harv}
\bibliographystyle{elsarticle-num}
\bibliography{GerchinovitzYu-journal,GerchinovitzYu-reference}

\end{document}